\documentclass[10pt,onecolumn,journal, draftclsnofoot]{IEEEtran}

\usepackage{hyperref} 

\usepackage{personal} 

\begin{document}

\title{Online Learning Schemes for Power Allocation in Energy Harvesting Communications}
\author{Pranav Sakulkar and Bhaskar Krishnamachari \\
Ming Hsieh Department of Electrical Engineering \\ 
Viterbi School of Engineering \\
University of Southern California, Los Angeles, CA, USA \\
{\tt \{sakulkar, bkrishna\}@usc.edu } 
\thanks{Part of this work was presented at the \it{International Conference on Signal Processing and Communications (SPCOM)}, June 2016.}
}
\date{\today}

\maketitle

\begin{abstract}
We consider the problem of power allocation over one or more time-varying channels with unknown distributions in energy harvesting communications. In the single-channel case, the transmitter chooses the transmit power based on the amount of stored energy in its battery with the goal of maximizing the average rate over time. We model this problem as a Markov decision process (MDP) with transmitter as the agent, battery status as the state, transmit power as the action and rate as the reward. The average reward maximization problem can be modelled by a linear program (LP) that uses the transition probabilities for the state-action pairs and their reward values to select a power allocation policy. This problem is challenging because the uncertainty in channels implies that the mean rewards associated with the state-action pairs are unknown. We therefore propose two online learning algorithms: LPSM and Epoch-LPSM that learn these rewards and adapt their policies over time. For both algorithms, we prove that their regret is upper-bounded by a constant. To our knowledge this is the first result showing constant regret learning algorithms for MDPs with unknown mean rewards. We also prove an even stronger result about LPSM: that its policy matches the optimal policy exactly in finite expected time. Epoch-LPSM incurs a higher regret compared to LPSM, while reducing the computational requirements substantially. We further consider a multi-channel scenario where the agent also chooses a channel in each slot, and present our multi-channel LPSM (MC-LPSM) algorithm that explores different channels and uses that information to solve the LP during exploitation. MC-LPSM incurs a regret that scales logarithmically in time and linearly in the number of channels. Through a matching lower bound on the regret of any algorithm, we also prove the asymptotic order optimality of MC-LPSM.
\end{abstract}


\section{Introduction}
Communication systems where the transmissions are powered by the energy harvested from nature have rapidly emerged as viable options for the next-generation wireless networks \cite{ulukus2015}. These advances are promising as they help us reduce the dependence on the conventional sources of energy and the carbon footprint of our systems. The ability to harvest energy promises self-sustainability and prolonged network lifetimes that are only limited by the communication hardware rather than the energy storage. Energy can be harvested from various natural sources, such as solar, thermal, chemical, biological etc. Their technologies differ in terms of their efficiency and harvesting capabilities, depending on the mechanisms, devices, circuits used. Since these energy sources are beyond human control, energy harvesting brings up a novel aspect of irregular and random energy source for communication. This demands a fresh look at the transmission schemes used by wireless networks. The next generation wireless networks need to be designed while keeping the irregularities and randomness in mind.

The performance of the energy harvesting communication systems is dependent on the efficient utilization of energy that is currently stored in the battery, as well as that is to be harvested over time. These systems must make decisions while keeping their impact on the future operations in mind. Such problems of optimal utilization of the available resources can be classified into offline optimization \cite{yang2012, tutuncuoglu2012, ozel2011} and online optimization \cite{lei2009, sinha2012, wang2012, ho2012, ozel2011} problems. In the offline optimization problems, the transmitter deterministically knows the exact amounts of the harvested energies and the data along with their exact arrival times. These assumptions are too optimistic, since the energy harvesting communication systems are usually non-deterministic. In the online optimization problems, however, the transmitter is assumed to know the distributions or some statistics of the harvesting and the data arrival processes. It may get to know their instantaneous realizations before making decisions in each slot. The problem considered in this paper falls in the category of the online optimization problems. 

In our paper, the channel gain distributions are assumed to be unknown and the harvested energy is assumed be stochastically varying with a known distribution. This is based on the fact that the weather conditions are more predictable than the radio frequency (RF) channels which are sensitive to time-varying multi-path fading. In the single channel case, the transmitter has to decide its transmit power level based on the current battery status with the goal maximizing the average expected transmission rate obtained over time. We model the system as a Markov decision process (MDP) with the battery status as the state, the transmit power as the action, the rate as the reward. The power allocation problem, therefore, reduces to the average reward maximization problem for an MDP. Since the channel gain distribution is unknown, their expected rates for different power levels are also unknown. The transmitter or the agent, therefore, cannot determine the optimal mapping from the battery state to the transmit power precisely. It needs to learn these rate values over time and make decisions along the way. We cast this problem as an online learning problem over an MDP.

In the multi-channel case, the agent needs to also select a channel from the set of channels for transmission. The gain-distributions of these channels are, in general, different and unknown to the agent. Due to different distributions, the optimal channels for different transmit powers can be different. The agent, therefore, needs to explore different channels over time to be able to estimate their rate values and use these estimates to choose a power-level and a channel at each time. This exploration of different channels leads the agent into making non-optimal power and channel decisions and thus hampers the performance of the online learning algorithm.  

One interesting feature of this problem is that the data-rate obtained during transmission is a known function of the chosen transmit power level and the instantaneous channel condition. Whenever a certain power is used for transmission, the agent can figure out the instantaneous channel condition once the instantaneous rate is revealed to it. This information about the instantaneous gain of the chosen channel can, therefore, be used to update the rate-estimates for all power levels. The knowledge of the rate function can be used to speed up the learning process in this manner.

\subsection{Contributions}
The problem of maximizing the average expected reward of an MDP can be formulated as a linear program (LP). The solution of this LP gives the stationary distribution over the state-action pairs under the optimal policy. If the MDP is ergodic, then there exists a deterministic optimal policy. We model the problem of communication over a single channel using the harvested energy as an MDP and prove its ergodicity. This helps us focus only on the deterministic state-action mappings which are finite in number.

The LP formulation helps us characterize the optimal policy that depends on the transition probabilities for the state-action pairs and their corresponding mean rewards. We use the optimal mean reward obtained from the LP as a benchmark to compare the performance of our algorithms with. Since the mean rewards associated the state-action pairs are unknown to the agent, we propose two online learning algorithms: LPSM and Epoch-LPSM that learn these rewards and adapt their policies along the way. The LPSM algorithm solves the LP at each step to decide its current policy based on its current sample mean estimates for the rewards, while the Epoch-LPSM algorithm divides the time into epochs, solves the LP only at the beginning of the epochs and follows the obtained policy throughout that epoch. We measure the performance of our online algorithms in terms of their regrets, defined as the cumulative difference between the optimal mean reward and the instantaneous reward of the algorithm. We prove that the reward loss or regret incurred by each of these algorithms is upper bounded by a constant. To our knowledge this is the first result where constant regret algorithms are proposed for the average reward maximization problem over MDPs with stochastic rewards with unknown means. We further prove that the LPSM algorithm starts following the genie's optimal policy in finite expected time. The finite expected time is an even stronger result than the constant regret guarantee. 

Our proposed single channel algorithms greatly differ in their computational requirements. Epoch-LPSM incurs a higher regret compared to LPSM, but reduces the computational requirements substantially. LPSM solves a total of $T$ LPs in time $T$, whereas Epoch-LPSM solves only $O(\ln T)$ number of LPs. We introduce two parameters $n_{0}$ and $\eta$ that reveal the computation vs regret tradeoff for Epoch-LPSM. Tuning these parameters allows the agent to control the system based on its performance requirements.  

We extend our framework to the case of multiple channels where the agent also needs to select the transmission channel in each slot. We present our MC-LPSM algorithm that deterministically separates exploration from exploitation. MC-LPSM explores different channels to learn their expected rewards and uses that information to solve the average reward maximization LP during the exploitation slots. The length of the exploration sequence scales logarithmically over time and contributes to the bulk of the regret. This exploration, however, helps us bound the exploitation regret by constant. We, therefore, prove a regret bound for MC-LPSM that scales logarithmically in time and linearly in the number of channels. This design of the exploration sequence, however, needs to know a lower bound on the difference in rates for the channels. We observe that this need of knowing some extra information about the system can be eliminated by using a longer exploration sequence as proposed in \cite{vakili2013}. The regret of this design can be made arbitrarily close to the logarithmic order. We also prove an asymptotic regret lower bound of $\Om(\ln T)$ for any algorithm under certain conditions. This proves the asymptotic order optimality of the proposed MC-LPSM approach. We further show that, similar to Epoch-LPSM, the MC-LPSM algorithm also solves only $O(\ln T)$ number of LPs in time $T$. 

We show that the proposed online learning algorithms also work for cost minimization problems in packet scheduling with power-delay tradeoff with minor changes.

\subsection{Organization}
This paper is organized as follows. First, we describe the model for the energy harvesting communication system using a single channel, formulate this problem as an MDP and discuss the structure of the optimal policy in section \ref{sec:sys_model}. We then propose our online learning algorithms LPSM and Epoch-LPSM for single channel systems and prove their regret bounds in section \ref{sec:online_learning}. We extend our approach to multi-channel systems, propose our MC-LPSM algorithm, and analyze its regret in section \ref{sec:multi_channel}. In section \ref{sec:packet}, we show that our online learning framework can also model the average cost minimization problems over MDPs.  Section \ref{sec:simu} presents the results of numerical simulations for this problem and section \ref{sec:conclusion} concludes the paper. We also include appendices \ref{apx:lemmas} and \ref{apx:MC} to discuss and prove some of the technical lemmas at the end of the paper.

\section{Related Work}
The offline optimization problems in the energy harvesting communications assume a deterministic system and the exact knowledge of energy and data arrival times and their amounts. In \cite{yang2012}, the goal is to minimize the time by which all the packets are delivered. In \cite{tutuncuoglu2012}, a finite horizon setting is considered with the goal of maximizing the amount of transmitted data. These two problems are proved to be duals of each other in \cite{tutuncuoglu2012}. In the online problems, the system is usually modelled as an MDP with the objective being the maximization of the average reward over time. In \cite{li2010}, the packets arrive according a Poisson distribution and each packet has a random value assigned to it. The reward, in this setting, corresponds to the sum of the values of the successfully transmitted packets. In \cite{sinha2012}, the transmitter is assumed to know the full channel state information before the transmission in each slot. The properties of the optimal transmission policy are characterized using dynamic programming for this setting. In \cite{wang2012}, the data arrivals are assumed to follow a Bernoulli distribution and a policy iteration based scheme is designed to minimize the transmission errors. In \cite{ho2012}, power allocation policies over a finite time horizon with known channel gain and harvested energy distributions are studied. In \cite{ozel2011}, the offline and online versions of the throughput optimization problem are studied for a fading channel.

Our problem can be seen from the lens of contextual bandits, which are extensions of the standard multi-armed bandits (MABs). In the standard MAB problem \cite{lai1985, auer2002, vakili2013}, the agent is presented with a set of arms each providing stochastic rewards over time with unknown distributions and it has to choose an arm in each trial with the goal of maximizing the sum reward over time. In \cite{lai1985}, Lai and Robbins provide an asymptotic lower bound of $\Om(\ln T)$ on the expected regret of any algorithm for this problem. In \cite{auer2002}, an upper confidence bound based policy called UCB1 and a randomized policy that separates exploration from exploitation called $\e$-greedy are proposed and are also proved to achieve logarithmic regret bounds for arm distributions with finite support. In \cite{vakili2013}, a deterministic equivalent of $\e$-greedy called DSEE is proposed and proved to provide similar regret guarantees. In the contextual bandits, the agent also sees some side-information before making its decision in each slot. In the standard contextual bandit problems \cite{langford2008epoch, dudik2011randUCB, agarwal2014monster, sakulkar2016dcb}, the contexts are assumed to be drawn from an unknown distribution independently over time. In our problem, the battery state can be viewed as the context. We model the context transitions by an MDP, since the agent's action at time $t$ affects not only the instantaneous reward but also the context in slot $t+1$. The agent, therefore, needs to decide the actions with the global objective in mind, i.e. maximizing the average reward over time. The algorithms presented in \cite{langford2008epoch, dudik2011randUCB, agarwal2014monster} do not assume any specific relation between the context and the reward. The DCB($\e$) algorithm presented in \cite{sakulkar2016dcb}, however, assumes that the mapping from the context and random instance to the reward is a known function and uses this function knowledge to reduce the expected regret. It must, however, be noted that the MDP formulation generalizes the contextual bandit setting in \cite{sakulkar2016dcb}, since the i.i.d. context case can be viewed as a single state MDP. 

Our problem is also closely related to the reinforcement learning problem over MDPs from \cite{ortner2007ucrl,auer2009ucrl2,tewari2008olp}. The objective for these problems is to maximize the average undiscounted reward over time. In \cite{ortner2007ucrl,auer2009ucrl2}, the agent is unaware of the transition probabilities and the mean rewards corresponding to the state-action pairs. In \cite{tewari2008olp}, the agent knows the mean rewards, but the transition probabilities are still unknown. In our problem, the mean rewards are unknown, while the transition probabilities of the MDP can be inferred from the knowledge of the arrival distribution and the action taken from each state. In contrast to the works above, for our problem motivated by the practical application in energy harvesting communications, we show that the learning incurs a constant regret in the single channel case. 

\section{System Model} \label{sec:sys_model}
We describe the model of the energy harvesting communication system considered in this paper using a single channel. Consider a time-slotted energy harvesting communication system where the transmitter uses the harvested power for transmission over a channel with stochastically varying channel gains with unknown distribution as shown in figure \ref{fig:en_harvest}. Let $p_{t}$ denote the harvested power in the $t$-th slot which is assumed to be i.i.d. over time. Let $Q_{t}$ denote the stored energy in the transmitter's battery that has a capacity of $Q_{\max}$. Assume that the transmitter decides to use $q_{t} (\leq Q_{t})$ amount of power for transmission in $t$-th slot. We assume discrete and finite number of power levels for the harvested and transmit powers. The rate obtained during the $t$-th slot is assumed to follow a relationship 
\begin{equation}
    r_{t} = B \log_{2}(1 + q_{t} X_{t}), \label{eq:rate_eq}
\end{equation}
where $X_{t}$ denotes the instantaneous channel gain-to-noise ratio of the channel which is assumed to be i.i.d. over time and $B$ is the channel bandwidth. The battery state gets updated in the next slot as 
\begin{equation}
    Q_{t+1} = \min\{Q_{t} - q_{t} + p_{t}, Q_{\max} \}. \label{eq:battery_update}
\end{equation}
The goal is utilize the harvested power and choose a transmit power $q_{t}$ in each slot sequentially to maximize the expected average rate $\underset{T\rightarrow \infty}{\lim} \frac{1}{T} \bbE \left[ \sum_{t=1}^{T}  r_{t} \right]$ obtained over time.

\begin{figure}
    \centering
    \includegraphics[width=0.8\textwidth]{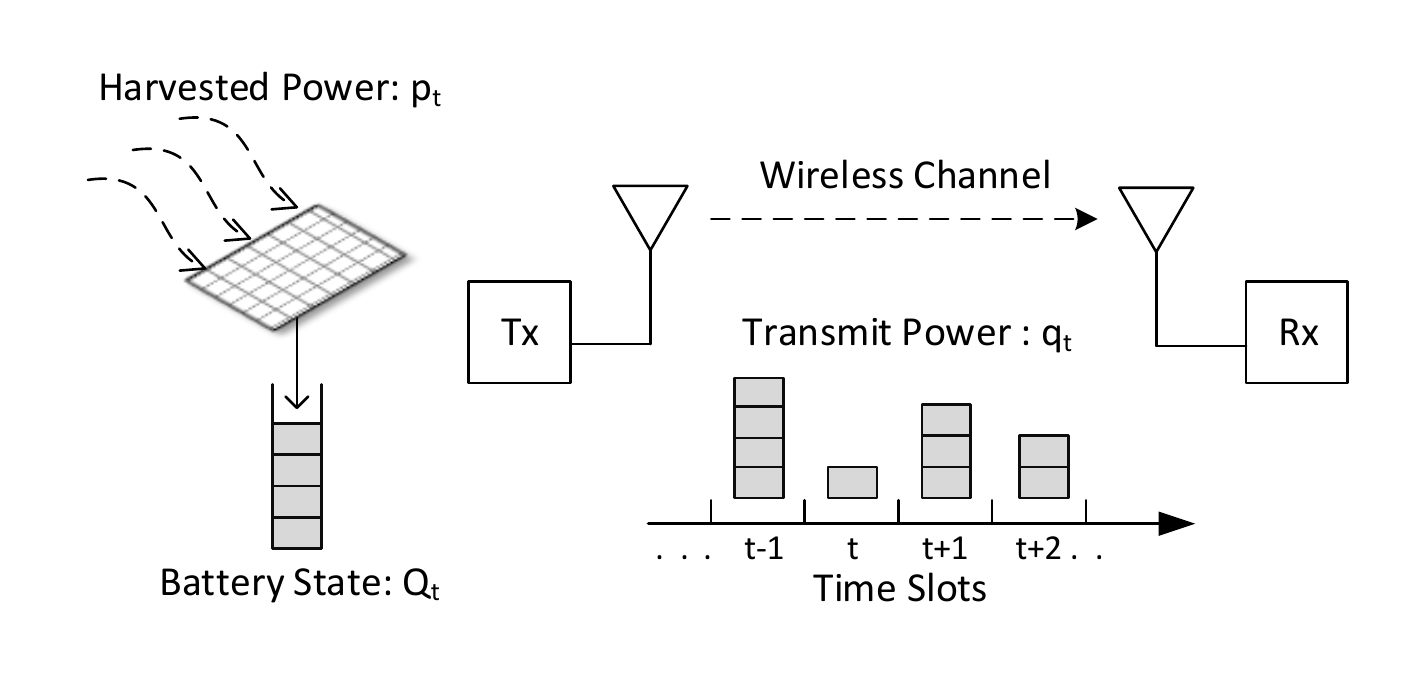}
    \caption{Power allocation over a wireless channel in energy harvesting communications}
    \label{fig:en_harvest}
\end{figure}

\subsection{Problem Formulation}
Consider an MDP $\sM$ with a finite state space $\sS$ and a finite action space $\sA$. Let $\sA_{s} \subset \sA$ denote the set of allowed actions from state $s$. When the agent chooses an action $a_{t} \in \sA_{s_{t}}$ in state $s_{t} \in \sS$, it receives a random reward $r_{t}(s_{t},a_{t})$. Based on the agent's decision the system undergoes a random transition to a state $s_{t+1}$ according to the transition probability $P(s_{t+1} \mid s_{t}, a_{t})$. In the energy harvesting problem, the battery status $Q_{t}$ represents the system state $s_{t}$ and the transmit power $q_{t}$ represents the action taken $a_{t}$ at any slot $t$. 

In this paper, we consider systems where the random rewards of various state action pairs can be modelled as
\begin{equation}
    r_{t}(s_{t},a_{t}) = f(s_{t}, a_{t}, X_{t}), \label{eq:reward_fn}
\end{equation}
where $f$ is a reward function known to the agent and $X_{t}$ is a random variable internal to the system that is i.i.d. over time. Note that in the energy harvesting communications problem, the reward is the rate obtained at each slot and the reward function is defined in equation (\ref{eq:rate_eq}). 
In this problem, the channel gain-to-noise ratio $X_{t}$ corresponds to the system's internal random variable. We assume that the distribution of the harvested energy $p_{t}$ is known to the agent. This implies that the state transition probabilities $P(s_{t+1} \mid s_{t}, a_{t})$ are inferred by the agent based on the update equation (\ref{eq:battery_update}). 

A policy is defined as any rule for choosing the actions in successive time slots. The action chosen at time $t$ may, therefore, depend on the history of previous states, actions and rewards. It may even be randomized such that the action $a \in \sA_{s}$ is chosen from some distribution over the actions. A policy is said to be stationary, if the action chosen at time $t$ is only a function of the system state at $t$. This means that a deterministic stationary policy $\b$ is a mapping from the state $s \in \sS$ to its corresponding action $a \in \sA_{s}$. When a stationary policy is played, the sequence of states $\{s_{t} \mid t =1,2,\cdots \}$ follows a Markov chain. An MDP is said to be {\it ergodic}, if every deterministic stationary policy leads to an irreducible and aperiodic Markov chain. According to section V.3 from \cite{ross1983sdp}, the average reward can be maximized by an appropriate deterministic stationary policy $\b^{*}$ for an ergodic MDP with finite state space. In order to arrive at an ergodic MDP for the energy harvesting communications problem, we make following assumptions. When the battery state $Q_{t} > 0$, the transmit power $q_{t} > 0$. The distribution of the harvested energy is such that $\pr\{ p_{t} = p\} > 0$ for all $0 \leq p \leq Q_{\max}$. Under these assumptions, we claim and prove the ergodicity of the MDP as follows.

\begin{proposition}
The MDP corresponding to the transmit power selection problem in energy harvesting communications is ergodic.
\end{proposition}
\begin{proof}
Consider any policy $\b$ and let $P^{(n)}(s, s')$ be the $n$-step transition probabilities associated with the Markov chain resulting from the policy. 

First, we prove that $P^{(1)}(s, s') > 0$ for any $s' \geq s$ as follows. According to the state update equations, 
\begin{equation}
s_{t+1} = s_{t} - \b(s_{t}) + p_{t}.
\end{equation}
The transition probabilities can, therefore, be expressed as
\begin{align}
P^{(1)}(s, s') = \pr \{ p = s' - s + \b(s) \} \geq 0, \label{eq:tran_prob_1}
\end{align}
since $s' \geq s$ and $\b(s) \geq 0$ for all states. This implies that any state $s' \in \sS$ is accessible from any other state $s$ in the resultant Markov chain, if $s \leq s'$.

Now, we prove that $P^{(1)}(s, s-1) > 0$ for all $s \geq 1$ as follows. From equation (\ref{eq:tran_prob_1}), we observe that 
\begin{align}
P^{(1)}(s, s-1) = \pr \{ p = \b(s) - 1 \} \geq 0, \label{eq:tran_prob_2}
\end{align}
since $\b(s) \geq 1$ for all $s \geq 1$. This implies that every state $s \in \sS$ is accessible from the state $s+1$ in the resultant Markov chain.

Equations (\ref{eq:tran_prob_1}) and (\ref{eq:tran_prob_2}) imply that all the state pairs $(s, s+1)$ communicate with each other. Since communication is an equivalence relationship, all the states communicate with each other and the resultant Markov chain is irreducible. Also, equation (\ref{eq:tran_prob_1}) implies that $P^{(1)}(s, s) > 0$ for all the states and the Markov chain is, therefore, aperiodic.
\end{proof}

Since the MDP under consideration is ergodic, we restrict ourselves to the set of deterministic stationary policies which we interchangeably refer to as policies henceforth. Let $\mu(s,a)$ denote the expected reward associated with the state-actions pair $(s,a)$ which can be expressed as
\begin{equation}
\mu(s,a) = \bbE \left[ r(s,a) \right] = \bbE_{X} \left[ f(s,a, X) \right]. \label{eq:mean_reward}
\end{equation}
For ergodic MDPs, the optimal mean reward $\rho^{*}$ is independent of the initial state (see \cite{puterman2005mdp}, section 8.3.3). It is specified as 
\begin{equation}
    \rho^{*} = \max_{\b \in \sB} \rho(\b, \bfM),
\end{equation}
where $\sB$ is the set of all policies, $\bfM$ is the matrix whose $(s,a)$-th entry is $\mu(s,a)$, and $\rho(\b, \bfM)$ is the average expected reward per slot using policy $\b$. We use the optimal mean reward as the benchmark and define the cumulative regret of a learning algorithm after $T$ time-slots as
\begin{equation}
    \mathfrak{R}(T) \defeq T \rho^{*} - \bbE\left[ \sum_{t=0}^{T-1} r_{t}\right]. \label{eq:regret}
\end{equation}

\subsection{Optimal Stationary Policy}
When the expected rewards for all state-action pairs $\mu(s,a)$ and the transition probabilities $P(s' \mid s,a)$ are known, the problem of determining the optimal policy to maximize the average expected reward over time can be formulated as a linear program (LP) (see e.g. \cite{ross1983sdp}, section V.3) shown below.
\begin{equation}
\begin{aligned}
    & \text{maximize}         & & \sum_{s \in \sS} \sum_{a \in \sA_{s}} \pi(s, a) \mu(s, a) \\
    & \text{subject to} & & \pi(s, a) \geq 0, \; \forall s \in \sS, a \in \sA_{s}, \\
    &                   & & \sum_{s \in \sS} \sum_{a \in \sA_{s}} \pi(s, a) = 1,  \\
    &                   & & \sum_{a \in \sA_{s'}} \pi(s', a) = \sum_{s \in \sS} \sum_{a \in \sA_{s}} \pi(s, a) P(s' \mid s, a), \; \forall s' \in \sS, 
\end{aligned} \label{eq:LP}
\end{equation}
where $\pi(s,a)$ denotes the stationary distribution of the MDP. The objective function of the LP from equation (\ref{eq:LP}) gives the average rate corresponding to the stationary distribution $\pi(s,a)$, while the constraints make sure that this stationary distribution corresponds to a valid policy on the MDP. Such LPs can be solved by using standard solvers such as CVXPY \cite{cvxpy}. 

If $\pi^{*}(s, a)$ is the solution to the LP from (\ref{eq:LP}), then for every $s \in \sS$, $\pi(s, a) > 0 $ for only one action $a \in \sA_{s}$. This is due to the fact the the optimal policy $\b^{*}$ is deterministic for ergodic MDPs in average reward maximization problems (see \cite{puterman2005mdp}, section 8.3.3). Thus for this problem, $\b^{*}(s) = \argmax\limits_{a \in \sA_{s}} \pi^{*}(s, a)$. Note that we, henceforth, drop the action index from the stationary distribution, since the policies under consideration are deterministic and the corresponding action is, therefore, deterministically known. In general, we use $\pi_{\b}(s)$ to denote the stationary distribution corresponding to the policy $\b$. It must be noted that the stationary distribution of any policy is independent of the reward values and only depends on the transition probability for every state-action pair. The expected average reward depends on the stationary distribution as 
\begin{equation}
    \rho(\b, \bfM) = \sum_{s \in \sS} \pi_{\b}(s) \mu(s, \b(s)).
\end{equation}
In terms of this notation, the LP from (\ref{eq:LP}) equivalent to maximizing $\rho(\b, \bfM)$ over $\b \in \sB$. Since the matrix $\bfM$ is unknown, we develop online learning policies for our problem in the next section.

\section{Online Learning Algorithms} \label{sec:online_learning}
For the power allocation problem under consideration, although the agent knows the state transition probabilities, the mean rewards for the state-action pairs $\mu(s,a)$ values are still unknown. Hence, the agent cannot solve the LP from (\ref{eq:LP}) to figure out the optimal policy. Any online learning algorithm needs to learn the reward values over time and update its policy adaptively. One interesting aspect of the problem, however, is that the reward function from equation (\ref{eq:reward_fn}) is known to the agent. Since the reward functions under consideration (\ref{eq:rate_eq}) is bijective, once the reward is revealed to the agent, it can infer the instantaneous realization of the random variable $X$. This can be used to predict the rewards that would have been obtained at that time for other state-action pairs using the function knowledge.

In our online learning framework, we store the average values of these inferred rewards $\th(s,a)$ for all state-action pairs. The idea behind our algorithms is to use the estimated sample mean values for the optimization problem instead of the unknown $\mu(s,a)$ values in the objective function of the LP from (\ref{eq:LP}). Since the $\th(s,a)$ values get updated after each reward revelation, the agent needs to solve the LP again and again. We propose two online learning algorithms: LPSM (linear program of sample means) where the agent solves the LP at each slot and Epoch-LPSM where the LP is solved at fixed pre-defined time slots. Although the agent is unaware of the actual $\mu(s,a)$ values, it learns the statistics $\th(s,a)$ over time and eventually figures out the optimal policy.

Let $B(s,a) \geq \sup_{x \in \sX} f(s,a, x) - \inf_{x \in \sX} f(s,a, x)$ denote any upper bound on the maximum possible range of the reward for the state-action pair $(s,a)$ over the support $\sX$ of the random variable $X$. We use following notations in the analysis of our algorithms: $B_{0} \defeq \underset{(s,a)}{\max}\; B(s,a)$, $\D_{1} \defeq \rho^{*} - \underset{\b \neq \b^{*}}{\max}\; \rho(\b, \bfM)$. The total number of states and actions are specified as $S \defeq \abs{\sS}$, $A \defeq \abs{\sA}$, respectively. Also, $\Th_{t}$ denotes the matrix containing the entries $\th_{t}(s,a)$ at time $t$.

\subsection{LPSM}
The LPSM algorithm presented in algorithm \ref{algo:LPSM} solves the LP at each time-step and updates its policy based on the solution obtained. It stores only one $\th$ value per state-action pair. Its required storage is, therefore, $O(SA)$. In theorem \ref{thm:non_opt_1}, we derive an upper bound on the expected number of slots where the LP fails to find the optimal solution during the execution of LPSM. We use this result to bound the total expected regret of LPSM in theorem \ref{thm:reg_bound_1}. These results guarantee that the regret is always upper bounded by a constant. Note that, for the ease of exposition, we assume that the time starts at $t=0$. This simplifies the analysis and has no impact on the regret bounds. 

\begin{algorithm}
\caption{LPSM}
\label{algo:LPSM}
\begin{algorithmic}[1]
\STATE  {\bf Initialization:} For all $(s,a)$ pairs, $\th(s,a) = 0$.
\FOR {$n = 0$} 
\STATE Given the state $s_{0}$ and choose any valid action;
\STATE Update all $(s,a)$ pairs: $\th(s,a) = f(s,a, x_{0})$;
\ENDFOR
\STATE // MAIN LOOP
\WHILE {1}
\STATE $n = n +1$;
\STATE Solve the LP from (\ref{eq:LP}) using $\th(s,a)$ in place of unknown $\mu(s,a)$;
\STATE In terms of the LP solution $\pi_{(n)}$, define $\b_{n}(s)= \argmax\limits_{a \in \sA_{s}} \pi_{(n)}(s,a), \; \forall s \in \sS$;
\STATE Given the state $s_{n}$, select the action $\b_{n}(s_{n})$;
\STATE Update for all valid $(s,a)$ pairs:
\begin{align}
\th(s,a) \leftarrow \frac{n \th(s,a) + f(s,a, x_{n})}{n+1}; \nn
\end{align} 
\ENDWHILE
\end{algorithmic}
\end{algorithm} 

\begin{theorem} \label{thm:non_opt_1}
The expected number of slots where non-optimal policies are played by LPSM is upper bounded by
\begin{equation}
1 + \frac{\left( 1 + A \right) S}{\me^{\frac{1}{2} \left(\frac{\D_{1}}{B_{0}}\right)^{2} } - 1}. \label{eq:UB_LPSM}
\end{equation}
\end{theorem}
\begin{proof}
Let $\b_{t}$ denote the policy obtained by LPSM at time $t$ and $\bbI(z)$ be the indicator function defined to be $1$ when the predicate $z$ is true, and $0$ otherwise. Now the number of slots where non-optimal policies are played can be expressed as
\begin{align}
N_{1} &= 1 + \sum_{t=1}^{\infty} \bbI \left\{ \b_{t} \neq \b^{*} \right\} \nn \\
&= 1 + \sum_{t=1}^{\infty} \bbI \left\{ \rho(\b^{*}, \Th_{t}) \leq \rho(\b_{t}, \Th_{t})\right\}. \label{eq:N_bound}
\end{align}
We observe that $\rho(\b^{*}, \Th_{t}) \leq \rho(\b_{t}, \Th_{t})$ implies that at least one of the following inequalities must be true:
\begin{align}
\rho(\b^{*}, \Th_{t}) &\leq \rho(\b^{*}, \bfM) - \frac{\D_{1}}{2} \label{eq:event1} \\
\rho(\b_{t}, \Th_{t}) &\geq \rho(\b_{t}, \bfM) + \frac{\D_{1}}{2} \label{eq:event2} \\
\rho(\b^{*}, \bfM) &< \rho(\b_{t}, \bfM) +  \D_{1}. \label{eq:event3}
\end{align}
Note that the event from condition (\ref{eq:event3}) can never occur, because of the definition of $\D_{1}$. Hence we upper bound the probabilities of the other two events. For the first event from condition (\ref{eq:event1}), we get
\begin{align}
\pr \left\{ \rho(\b^{*}, \Th_{t}) \leq \rho(\b^{*}, \bfM) - \frac{\D_{1}}{2} \right\} &= \pr \left\{ \sum_{s \in \sS} \pi^{*}(s, \b^{*}(s)) \th_{t}(s, \b^{*}(s)) \leq \sum_{s \in \sS} \pi^{*}(s, \b^{*}(s)) \mu(s, \b^{*}(s)) - \frac{\D_{1}}{2} \right\}  \nn \\
&\leq \pr \left\{ \text{For at least one state}\ s \in \sS: \right. \nn \\
& \quad \quad \quad \left. \pi^{*}(s, \b^{*}(s)) \th_{t}(s, \b^{*}(s)) \leq \pi^{*}(s, \b^{*}(s)) \left( \mu(s, \b^{*}(s)) - \frac{\D_{1}}{2} \right) \right\}  \nn \\
&\leq \sum_{s \in \sS} \pr \left\{ \pi^{*}(s, \b^{*}(s)) \th_{t}(s, \b^{*}(s)) \leq \pi^{*}(s, \b^{*}(s)) \left( \mu(s, \b^{*}(s)) - \frac{\D_{1}}{2} \right) \right\}  \nn \\
&\leq \sum_{s \in \sS} \pr \left\{ \th_{t}(s, \b^{*}(s)) \leq \mu(s, \b^{*}(s)) - \frac{\D_{1}}{2}\right\} \nn \\
&\overset{(a)}{\leq} \sum_{s \in \sS} \me^{-2 \left(\frac{\D_{1}}{2 B(s,\b^{*}(s))}\right)^{2} t } \nn \\
&= S \me^{-2 \left(\frac{\D_{1}}{2 B_{0}}\right)^{2} t }, \label{eq:cond1_proof}
\end{align}
where $(a)$ holds due to Hoeffding's inequality from lemma \ref{lem:hoeffding} (see appendix \ref{apx:lemmas}).

Similarly for the second event from condition (\ref{eq:event2}), we get
\begin{align}
\pr \left\{ \rho(\b_{t}, \Th_{t}) \geq \rho(\b_{t}, \bfM) + \frac{\D_{1}}{2}\right\} &= \pr \left\{ \sum_{s \in \sS} \sum_{a \in \sA_{s}} \pi_{\b_{t}}(s, a) \th_{t}(s, a) \geq \sum_{s \in \sS} \sum_{a \in \sA_{s}} \pi_{\b_{t}}(s, a) \mu(s, a) + \frac{\D_{1}}{2} \right\}  \nn \\
&\leq \pr \left\{ \text{For at least one state-action pair}\ (s,a): \right. \nn \\
& \quad \quad \quad \left. \pi_{\b_{t}}(s, a) \th_{t}(s, a) \geq \pi_{\b_{t}}(s, a) \left( \mu(s, a) + \frac{\D_{1}}{2} \right) \right\}  \nn \\
&\leq \sum_{s \in \sS} \sum_{a \in \sA_{s}} \pr \left\{ \pi_{\b_{t}}(s, a) \th_{t}(s, a) \geq \pi_{\b_{t}}(s, a) \left(\mu(s, a) + \frac{\D_{1}}{2} \right) \right\}  \nn \\
&= \sum_{s \in \sS} \sum_{a \in \sA_{s}} \pr \left\{ \th_{t}(s, a) \geq \mu(s, a) + \frac{\D_{1}}{2} \right\} \nn \\
&\overset{(b)}{\leq} \sum_{s \in \sS} \sum_{a \in \sA_{s}} \me^{-2 \left(\frac{\D_{1}}{2 B(s,a)}\right)^{2} t } \nn \\
&\leq S A \me^{-2 \left(\frac{\D_{1}}{2 B_{0}}\right)^{2} t }, \label{eq:cond2_proof}
\end{align}
where $(b)$ holds due to Hoeffding's inequality from lemma \ref{lem:hoeffding} in appendix \ref{apx:lemmas}.

The expected number of non-optimal policies from equation (\ref{eq:N_bound}), therefore, can be expressed as
\begin{align}
\bbE[ N_{1}] & \leq 1 + \sum_{t=0}^{\infty} \pr \left\{ \rho(\b^{*}, \Th_{t}) \leq \rho(\b_{t}, \Th_{t})\right\} \nn \\
& \leq 1 + \sum_{t=1}^{\infty} \left( \pr \left\{ \rho(\b^{*}, \Th_{t}) \leq \rho(\b^{*}, \bfM) - \frac{\D_{1}}{2} \right\} + \pr \left\{ \rho(\b_{t}, \Th_{t}) \geq \rho(\b_{t}, \bfM) + \frac{\D_{1}}{2} \right\} \right) \nn \\
& \leq 1 + \left( 1 + A \right) S \sum_{t=1}^{\infty}  \me^{-2 \left(\frac{\D_{1}}{2 B_{0}}\right)^{2} t } \nn \\
& \leq 1 + \left( 1 + A \right) S \frac{\me^{-\frac{1}{2} \left(\frac{\D_{1}}{B_{0}}\right)^{2}}}{1 - \me^{-\frac{1}{2} \left(\frac{\D_{1}}{B_{0}}\right)^{2} } } \nn \\
&\leq 1 + \frac{\left( 1 + A \right) S}{\me^{\frac{1}{2} \left(\frac{\D_{1}}{B_{0}}\right)^{2} } - 1}. \label{eq:fail_count}
\end{align}
\end{proof}

It is important to note that even if the optimal policy is found by the LP and played during certain slots, it does not mean that regret contribution of those slots is zero.  According to the definition of regret from equation (\ref{eq:regret}), regret contribution of a certain slot is zero if and only if the optimal policy is played and the corresponding Markov chain is at its stationary distribution. In appendix \ref{apx:MC}, we introduce tools to analyze the mixing of Markov chains and characterize this regret contribution in theorem \ref{thm:opt_regret}. These results are used to upper bound the LPSM regret in the next theorem.

\begin{theorem} \label{thm:reg_bound_1}
The total expected regret of the LPSM algorithm is upper bounded by 
\begin{equation}
    \left(1 + \frac{\left( 1 + A \right) S}{\me^{\frac{1}{2} \left(\frac{\D_{1}}{B_{0}}\right)^{2} } - 1} \right) \left( \frac{\mu_{\max}}{1-\g} + \D_{\max} \right). \label{eq:regret_1}
\end{equation}
where $\g = \max\limits_{s, s' \in \sS} \norm{P_{*}(s',\cdot) - P_{*}(s,\cdot)}_{\tv}$, $P_{*}$ denotes the transition probability matrix corresponding to the optimal policy, $\mu_{\max} = \max\limits_{s \in \sS, a \in \sA_{s}} \mu(s,a)$ and $\D_{\max} =  \rho^{*} - \min\limits_{s \in \sS, a \in \sA_{s}} \mu(s,a)$.
\end{theorem}
\begin{proof}
The regret of LPSM arises when either non-optimal actions are taken or when optimal actions are taken and the corresponding Markov chain is not at stationarity. For the first source of regret, it is sufficient to analyze the number of instances where the LP fails to find the optimal policy. For the second source, however, we need to analyze the total number of phases where the optimal policy is found in succession.

Since only the optimal policy is played in consecutive slots in a phase, it corresponds to transitions on the Markov chain associated with the optimal policy and the tools from appendix \ref{apx:MC} can be applied. According to theorem \ref{thm:opt_regret}, the regret contribution of any phase is bounded from above by $(1-\g)^{-1} \mu_{\max}$. As proved in theorem \ref{thm:non_opt_1}, for $t\geq 1$, the expected number of instances of non-optimal policies is upper bounded by $\frac{\left( 1 + A \right) S}{\me^{\frac{1}{2} \left(\frac{\D_{1}}{B_{0}}\right)^{2} } - 1}$. Since any two optimal phases must be separated by at least one non-optimal slot, the expected number of optimal phases is upper bounded by $1 + \frac{\left( 1 + A \right) S}{\me^{\frac{1}{2} \left(\frac{\D_{1}}{B_{0}}\right)^{2} } - 1}$. Hence, for $t\geq 1$, the expected regret contribution from the slots following the optimal policy is upper bounded by 
\begin{equation}
    \left(1 + \frac{\left( 1 + A \right) S}{\me^{\frac{1}{2} \left(\frac{\D_{1}}{B_{0}}\right)^{2} } - 1} \right) \frac{\mu_{\max}}{1-\g}.
\end{equation}

Note that the maximum regret possible during one slot is $\D_{\max}$. Hence for the slots where non-optimal policies are played, the corresponding expected regret contribution is upper bounded by $\left(1 + \frac{\left( 1 + A \right) S}{\me^{\frac{1}{2} \left(\frac{\D_{1}}{B_{0}}\right)^{2} } - 1} \right)$.

Overall expected regret for the LPSM algorithm is, therefore, bounded from above by equation (\ref{eq:regret_1}).
\end{proof}

\begin{remark}
It must be noted that we call two policies as same if and only if they recommend identical actions for every state. It is, therefore, possible for a non-optimal policy to recommend optimal actions for some of the states. In the analysis of LPSM, we count all occurrences of non-optimal policies as regret contributing occurrences in order to upper bound the regret.
\end{remark}

\begin{remark}
Note that the LPSM algorithm presented above works for general reward functions $f$. The rate function in the energy harvesting communications is, however, not dependent on the state which is the battery status. It is a function of the transmit power level and the channel gain only. The LPSM algorithm, therefore, needs to store only one $\th$ variable for each transmit power level and needs $O(A)$ storage overall. The probability of event from condition (\ref{eq:event2}) is bounded by a tighter upper bound of $A \me^{-2 \left(\frac{\D_{1}}{2 B_{0}}\right)^{2} t }$. The regret upper bound from theorem \ref{thm:reg_bound_1} is also tightened to 
\begin{equation}
    \left(1 + \frac{\left( S + A \right)}{\me^{\frac{1}{2} \left(\frac{\D_{1}}{B_{0}}\right)^{2} } - 1} \right) \left( \frac{\mu_{\max}}{1-\g} +  \D_{\max} \right).
\end{equation}
\end{remark}

For the LPSM algorithm, we can prove a stronger result about the convergence time. Let $Z$ be the random variable corresponding to the first time-slot after which LPSM never fails to find the optimal policy. This means that $Z-1$ represents the last time-slot where LPSM finds a non-optimal policy. We prove that the expected value of $Z$ is finite, which means that LPSM takes only a finite amount time in expectation before it starts following the genie. We present this result in theorem \ref{thm:finite_exp_time}.

\begin{theorem} \label{thm:finite_exp_time}
For the LPSM algorithm, the expected value of the convergence time $Z$ is finite.
\end{theorem}
\begin{proof}
Since $Z-1$ denotes the index of the last slot where LPSM errs, all the slots from $Z$ onward must have found the optimal policy $\b^{*}$. We use this idea to bound the following probability
\begin{align}
\pr\{Z \leq Z_{0}\} &= \pr \{ \text{LPSM finds the optimal policy in all slots $Z_{0}$, $Z_{0}+1$, $\cdots$} \} \nn \\
&= 1 - \pr \left\{ \text{LPSM fails in at least one slot in $\{Z_{0}, Z_{0}+1, \cdots \}$} \right\} \nn \\
&\geq 1 - \sum_{t = Z_{0}}^{\infty} \pr \left\{ \text{LPSM fails at $t$} \right\} \nn \\
&\geq 1 - \sum_{t = Z_{0}}^{\infty} \left( 1 + A \right) S \me^{-2 \left(\frac{\D_{1}}{2 B_{0}}\right)^{2} t } \tag{From equation (\ref{eq:fail_count})} \\
&= 1 - \frac{(1+A)S \me^{-\frac{1}{2} \left(\frac{\D_{1}}{B_{0}}\right)^{2} Z_{0} }}{1 - \me^{-\frac{1}{2} \left(\frac{\D_{1}}{B_{0}}\right)^{2} }}.
\end{align}
We, therefore, get the following exponential inequality for $Z_{0} \geq 1$
\begin{align}
\pr\{Z > Z_{0}\} \leq \frac{(1+A)S \me^{-\frac{1}{2} \left(\frac{\D_{1}}{B_{0}}\right)^{2} Z_{0} }}{1 - \me^{-\frac{1}{2} \left(\frac{\D_{1}}{B_{0}}\right)^{2} }}.
\end{align}
The expectation of $Z$ can, now, be bounded as
\begin{align}
\bbE[Z] &= \sum_{Z_{0} = 0}^{\infty} Z_{0} \pr \{ Z = Z_{0}\} \nn \\
&= \sum_{Z_{0} = 0}^{\infty} \pr \{ Z > Z_{0}\} \nn \\
&\leq 1 + \sum_{Z_{0} = 1}^{\infty} \frac{(1+A)S \me^{-\frac{1}{2} \left(\frac{\D_{1}}{B_{0}}\right)^{2} Z_{0} }}{1 - \me^{-\frac{1}{2} \left(\frac{\D_{1}}{B_{0}}\right)^{2} }} \nn \\
&= 1 + \frac{(1+A)S \me^{-\frac{1}{2} \left(\frac{\D_{1}}{B_{0}}\right)^{2} }}{\left(1 - \me^{-\frac{1}{2} \left(\frac{\D_{1}}{B_{0}}\right)^{2} }\right)^{2}}.
\end{align}
The expected value of $Z$ is, therefore, finite.
\end{proof}

\begin{remark}
Note that the result about finite expected convergence time is not directly implied by the constant regret result. The proof of theorem \ref{thm:finite_exp_time} relies on the exponential nature of the concentration bound, whereas it is possible to prove constant expected regrets even for weaker concentration bounds \cite{sakulkar2016spcom}.
\end{remark}

\subsection{Epoch-LPSM}
The main drawback of the LPSM algorithm is that it is computationally heavy as it solves one LP per time-slot. In order to reduce the computation requirements, we propose the Epoch-LPSM algorithm in algorithm \ref{algo:Epoch-LPSM}. Epoch-LPSM solves the LP in each of the first $n_{0}$ slots, divides the later time into several epochs and solves the LPs only at the beginning of each epoch. The policy obtained by solving the LP at the beginning of an epoch is followed for the remaining slots in that epoch. We increase the length of these epochs exponentially as time progresses and our confidence on the obtained policy increases. In spite of solving much fewer number of LPs, the regret of Epoch-LPSM is still bounded by a constant. First, we obtain an upper bound on the number of slots where the algorithm plays non-optimal policies in theorem \ref{thm:non_opt_2} and later use this result to bound the regret in theorem \ref{thm:reg_bound_2}.

\begin{algorithm}
\caption{Epoch-LPSM}
\label{algo:Epoch-LPSM}
\begin{algorithmic}[1]
\STATE  {\bf Parameters:} $n_{0} \in \bbN$ and $\eta \in \{2, 3, \cdots \}$.
\STATE  {\bf Initialization:} $k = 0$, $n=0$ and for all $(s,a)$ pairs, $\th(s,a) = 0$.
\WHILE {$n < n_{0}$}
\STATE Follow LPSM algorithm to decide action $a_{n}$, update the $\th$ variables accordingly and increment $n$;
\ENDWHILE
\WHILE { $n \geq n_{0}$}
\STATE $n = n +1$;
\IF {$n = n_{0} \eta^{k}$}
\STATE $k = k+1$;
\STATE Solve the LP from (\ref{eq:LP}) with $\th(s,a)$ in place of unknown $\mu(s,a)$;
\STATE In terms of the LP solution $\pi_{(n)}$, define $\b_{(k)}(s)= \argmax\limits_{a \in \sA_{s}} \pi_{(n)}(s,a), \; \forall s \in \sS$;
\ENDIF
\STATE Given the state $s_{n}$, select the action $\b_{(k)}(s_{n})$;
\STATE Update for all $(s,a)$ pairs:
\begin{align}
\th(s,a) \leftarrow \frac{n \th(s,a) + f(s,a, x_{n})}{n+1}; \nn 
\end{align} 
\ENDWHILE
\end{algorithmic}
\end{algorithm} 

\begin{theorem} \label{thm:non_opt_2}
The expected number of slots where non-optimal policies are played by Epoch-LPSM is upper bounded by
\begin{equation}
1 + (1 + A) S \left( \frac{1 - \me^{-\frac{1}{2} \left(\frac{\D_{1}}{B_{0}}\right)^{2} n_{0}}}{\me^{\frac{1}{2} \left(\frac{\D_{1}}{B_{0}}\right)^{2} } - 1} \right) + (\eta - 1) (1 + A) S n_{0} \s_{n_{0}, \eta},
\end{equation}
where $\s_{n_{0}, \eta} = \sum_{k=0}^{\infty} \eta^{k} \me^{-\frac{1}{2} \left(\frac{\D_{1}}{B_{0}}\right)^{2} n_{0} \eta^{k} } < \infty$.
\end{theorem}
\begin{proof}
Note that epoch $k$ starts at $t=n_{0}\eta^{k-1}$ and end at $t=n_{0}\eta^{k}-1$. The policy obtained at $t=n_{0}\eta^{k-1}$ by solving the LP is, therefore, played for $(\eta^{k}-\eta^{k-1})n_{0}$ number of slots. 

Let us analyse the probability that the policy played during epoch $k$ is not optimal. Let that policy be $\b_{(k)}$.
\begin{align}
\pr\{ \b_{(k)} \neq \b^{*} \} &= \pr \{ \b_{n_{0} \eta^{k-1}} \neq \b^{*} \} \nn \\
&\overset{(a)}{\leq} (1 + A) S \me^{-2 \left(\frac{\D_{1}}{2 B_{0}}\right)^{2} n_{0} \eta^{k-1} }, \label{eq:epoch_prob}
\end{align}
where $(a)$ holds for all $k \geq 1$ as shown in the proof of theorem \ref{thm:non_opt_1}.

When the LP fails to obtain the optimal policy at the beginning of an epoch, then all the slots in that epoch will play the obtained non-optimal policy. Let $N_{2}$ denote total number of such slots. We get
\begin{align}
N_{2} = 1 + \sum_{t=1}^{n_{0}-1} \bbI \{ \b_{t} \neq \b^{*} \} + \sum_{k=1}^{\infty} n_{0}(\eta^{k}-\eta^{k-1}) \bbI \{ \b_{(k)} \neq \b^{*} \}. \nn
\end{align}
In expectation, we get
\begin{align}
\bbE [N_{2}] &= 1 + \sum_{t=1}^{n_{0}-1} \pr \{ \b_{t} \neq \b^{*} \} + \sum_{k=1}^{\infty} n_{0}(\eta^{k}-\eta^{k-1}) \pr \{ \b_{(k)} \neq \b^{*} \} \nn \\
&\leq 1 + \sum_{t=1}^{n_{0}-1} (1 + A) S \me^{-2 \left(\frac{\D_{1}}{2 B_{0}}\right)^{2} t } + \sum_{k=1}^{\infty} n_{0} (\eta^{k}-\eta^{k-1}) (1 + A) S \me^{-2 \left(\frac{\D_{1}}{2 B_{0}}\right)^{2} n_{0} \eta^{k-1} } \nn \\
&\leq 1 + (1 + A) S \sum_{t=1}^{n_{0}-1} \me^{-\frac{1}{2} \left(\frac{\D_{1}}{B_{0}}\right)^{2} t } + n_{0}(\eta - 1) (1 + A) S \sum_{k=0}^{\infty} \eta^{k} \me^{-\frac{1}{2} \left(\frac{\D_{1}}{B_{0}}\right)^{2} n_{0} \eta^{k} } \nn \\
&\leq 1 + (1 + A) S \me^{-\frac{1}{2} \left(\frac{\D_{1}}{B_{0}}\right)^{2} } \frac{1 - \me^{-\frac{1}{2} \left(\frac{\D_{1}}{B_{0}}\right)^{2} n_{0}}}{1 - \me^{-\frac{1}{2} \left(\frac{\D_{1}}{B_{0}}\right)^{2} }} + n_{0}(\eta - 1) (1 + A) S \s_{n_{0}, \eta},
\end{align}
where $\s_{n_{0}, \eta} < \infty$ holds due to ratio test for series convergence as
\begin{align}
\lim_{k \rightarrow \infty} \left|\frac{\eta^{k+1} \me^{-\frac{1}{2} \left(\frac{\D_{1}}{B_{0}}\right)^{2} n_{0} \eta^{k+1}}}{\eta^{k} \me^{-\frac{1}{2} \left(\frac{\D_{1}}{B_{0}} \right)^{2} n_{0} \eta^{k}}}\right| &= \lim_{k \rightarrow \infty} \left|\eta \me^{-\frac{1}{2} \left(\frac{\D_{1}}{B_{0}}\right)^{2} n_{0} (\eta^{k+1} - \eta^{k})}\right| \nn \\
&= \lim_{k \rightarrow \infty} \left| \eta \left(\me^{\frac{1}{2} \left(\frac{\D_{1}}{B_{0}}\right)^{2} n_{0} (\eta - 1) }\right)^{-\eta^{k}} \right| \nn \\
&= 0. \nn
\end{align}
\end{proof}

Now we analyse the regret of Epoch-LPSM in the following theorem.
\begin{theorem} \label{thm:reg_bound_2}
The total expected regret of the Epoch-LPSM algorithm is upper bounded by 
\begin{multline}
    \left( 1 + (1 + A) S \left( \frac{1 - \me^{-\frac{1}{2} \left(\frac{\D_{1}}{B_{0}}\right)^{2} n_{0}}}{\me^{\frac{1}{2} \left(\frac{\D_{1}}{B_{0}}\right)^{2} } - 1} \right) \right) \left( \frac{\mu_{\max}}{1-\g} + \D_{\max} \right) + (\eta - 1) (1 + A) S n_{0} \s_{n_{0}, \eta} \D_{\max} \\
    \quad\quad\quad \quad\quad\quad  + (1 + A) S \left( \sum_{k=1}^{\infty} \me^{-2 \left(\frac{\D_{1}}{2 B_{0}}\right)^{2} n_{0} \eta^{k-1} } \right) \frac{\mu_{\max}}{1-\g}. \label{eq:regret_2}
\end{multline}
\end{theorem}
\begin{proof}
First we analyse the regret contribution from the first $n_{0}$ slots. As argued in the proof of theorem \ref{thm:reg_bound_1}, the regret contribution of the first $n_{0}$ slots is upper bounded by 
\begin{equation}
    \left( 1 + (1 + A) S \left( \frac{1 - \me^{-\frac{1}{2} \left(\frac{\D_{1}}{B_{0}}\right)^{2} n_{0}}}{\me^{\frac{1}{2} \left(\frac{\D_{1}}{B_{0}}\right)^{2} } - 1} \right) \right) \left( \frac{\mu_{\max}}{1-\g} + \D_{\max} \right).
\end{equation}

Now we analyse the number of phases where the optimal policy is played in successive slots for $t \geq n_{0}$ . Note that any two optimal phases are separated by at least one non-optimal epoch. We bound the number of non-optimal epochs $N_{3}$ as
\begin{align}
\bbE[ N_{3} ] &= \sum_{k=1}^{\infty} \pr\{ \b_{(k)} \neq \b^{*} \} \nn \\
&\leq (1 + A) S \sum_{k=1}^{\infty} \me^{-2 \left(\frac{\D_{1}}{2 B_{0}}\right)^{2} n_{0} \eta^{k-1} } \tag{From equation (\ref{eq:epoch_prob})} \\
&< 0, \nn
\end{align}
where the series $\sum_{k=1}^{\infty} \me^{-2 \left(\frac{\D_{1}}{2 B_{0}}\right)^{2} n_{0} \eta^{k-1} }$ converges due to ratio test as
\begin{align}
\lim_{k \rightarrow \infty} \left|\frac{\me^{-\frac{1}{2} \left(\frac{\D_{1}}{B_{0}}\right)^{2} n_{0} \eta^{k+1}}}{\me^{-\frac{1}{2} \left(\frac{\D_{1}}{B_{0}} \right)^{2} n_{0} \eta^{k}}}\right| &= \lim_{k \rightarrow \infty} \left| \left(\me^{\frac{1}{2} \left(\frac{\D_{1}}{B_{0}}\right)^{2} n_{0} (\eta - 1) }\right)^{-\eta^{k}} \right| = 0. \nn
\end{align}
Hence for $t \geq n_{0}$, there can be at most $\bbE[N_{3}]$ number of optimal phases in expectation. Since each of these phases can contribute a maximum of $(1-\g)^{-1} \mu_{\max}$ regret in expectation, total regret from slots with optimal policies for $t \geq n_{0}$ is upper bounded by $ \bbE[N_{3}] \frac{\mu_{\max}}{1-\g}$. Also, the expected number of slots where a non-optimal policy is played for $t \geq n_{0}$ is bounded by $(\eta - 1) (1 + A) S n_{0} \s_{n_{0}, \eta}$ as derived in the proof of theorem \ref{thm:non_opt_2}. The regret contribution of these slots is, therefore, bounded by $(\eta - 1) (1 + A) S n_{0} \s_{n_{0}, \eta} \D_{\max}$, since the maximum expected regret incurred during any slot is $\D_{\max}$.

The total expected regret of Epoch-LPSM is, therefore, upper bounded by the expression (\ref{eq:regret_2}).
\end{proof}

\begin{remark}
Similar to the LPSM algorithm, the algorithm presented above works for general reward functions $f$. Since the rate function in the energy harvesting communications is not dependent on the state, the Epoch-LPSM algorithm needs to store only one $\th$ variable per transmit power level and uses $O(A)$ storage overall. The regret upper bound from theorem \ref{thm:reg_bound_2} is also tightened to 
\begin{multline}
    \left( 1 + (S + A) \left( \frac{1 - \me^{-\frac{1}{2} \left(\frac{\D_{1}}{B_{0}}\right)^{2} n_{0}}}{\me^{\frac{1}{2} \left(\frac{\D_{1}}{B_{0}}\right)^{2} } - 1} \right) \right) \left( \frac{\mu_{\max}}{1-\g} + \D_{\max} \right) + (\eta - 1) (S + A) n_{0} \s_{n_{0}, \eta} \D_{\max} \\
    \quad\quad\quad \quad\quad\quad  + (S + A) \left( \sum_{k=1}^{\infty} \me^{-2 \left(\frac{\D_{1}}{2 B_{0}}\right)^{2} n_{0} \eta^{k-1} } \right) \frac{\mu_{\max}}{1-\g}.
\end{multline}
\end{remark}

\subsection{Regret vs Computation Tradeoff}
The LPSM algorithm solves $T$ LPs in time $T$, whereas the Epoch-LPSM algorithm solves $n_{0}$ LPs in the initial $n_{0}$ slots and $\lceil \log_{\eta} (T - n_{0}) \rceil$ LPs when the time gets divided into epochs. This drastic reduction in the required computation comes at the cost of an increase in the regret for Epoch-LPSM. It must, however, be noted that both the algorithms have constant-bounded regrets. Also, increasing the value of the parameter $\eta$ in Epoch-LPSM leads to reduction in the number of LPs to be solved over time by increasing the epoch lengths. Any non-optimal policy found by LP, therefore, gets played over longer epochs increasing the overall regret. Increasing $n_{0}$ increases the total number of LPs solved by the algorithm while reducing the expected regret. The system designer can analyse the regret bounds of these two algorithms and its own performance requirements to choose the parameters $n_{0}$ and $\eta$ for the system. We analyse the effect of variation of these parameters on the regret performance of Epoch-LPSM through numerical simulations in section \ref{sec:simu}.

\section{Multi-Channel Communication} \label{sec:multi_channel}
In this section, we extend the energy harvesting communications problem to consider a system where there exists a set of parallel channels, with unknown statistics, for communication and one of these channels is to be selected in each slot. The goal is to utilize the battery at the transmitter and maximize the amount of data transmitted over time. Given a time-slotted system, we assume that the agent is aware of the distribution of energy arrival. The agent sees the current state of the battery and needs to decide the transmit power-level and the channel to be used for transmission. This problem, therefore, involves an additional decision making layer compared to the single channel case. Note that we use the terms transmit power and action interchangeably in this section.

For this problem, we simplify the notations used previously and drop the state as a parameter for the reward function $f$, since the rate is not a function of the battery state and only depends on the transmit power-level and the channel gain. These channels will, in general, have different distributions of channel gains and there may not be a single channel that is optimal for all transmit power-levels. The expected rate achieved by selecting $j$-th channel from the set of $M$ channels and an action $a$ corresponding the transmit power used is
\begin{align}
\mu_{j}(a) = \bbE_{X_{j}} [f(a, X_{j})],
\end{align}
where $X_{j}$ denotes the random gain of $j$-th channel. Let us define $\phi^{*}: \sA \rightarrow \{1,2 \cdots, M\}$ as the mapping from transmit power-levels to their corresponding optimal channels:
\begin{align}
 \phi^{*}(a) = \argmax\limits_{j \in \{1, 2\cdots , M\}} \mu_{j}(a).
\end{align}
A genie that knows the distributions of different channels gains can figure out the optimal channel mapping $\phi^{*}$. Once an action $a$ is chosen by the genie, there is no incentive to use any channel other than $\phi^{*}(a)$ for transmission during that slot. Let $\mu^{*}(a)$ denote the expect rate corresponding the best channel for action $a$, i.e. $\mu^{*}(a) = \mu_{\phi^{*}(a)}(a)$. The genie uses these values to solve the following LP. 
\begin{equation}
\begin{aligned}
    & \text{maximize}         & & \sum_{s \in \sS} \sum_{a \in \sA_{s}} \pi(s, a) \mu^{*}(a) \\
    & \text{subject to} & & \pi(s, a) \geq 0, \; \forall s \in \sS, a \in \sA_{s}, \\
    &                   & & \sum_{s \in \sS} \sum_{a \in \sA_{s}} \pi(s, a) = 1,  \\
    &                   & & \sum_{a \in \sA_{s'}} \pi(s', a) = \sum_{s \in \sS} \sum_{a \in \sA_{s}} \pi(s, a) P(s' \mid s, a), \; \forall s' \in \sS, 
\end{aligned} \label{eq:MC_LP}
\end{equation}
The genie obtains $\b^{*}: \sS \rightarrow \sA$, the optimal mapping from the battery state to the transmit power-level using the non-zero terms of the optimal stationary distribution $\pi^{*}(s,a)$. Note that the constraints of the optimization problem (\ref{eq:MC_LP}) ensure that the stationary distribution actually corresponds to some valid deterministic state-action mapping.

Let $\sB$ be the set of all state-action mappings. There are only a finite number of such mappings $\b \in \sB$ and the stationary distribution only depends on the matrix of state transition probabilities which is assumed to be known. We use $\pi_{\b}(s)$ denote the stationary distribution corresponding to the state-action mapping $\b$. We dropped the action parameter from the previous notation, since it is implicit from $\b$. The expected average reward of the power selection policy $\b$ along with a channel selection policy $\phi$ is calculated as
\begin{align}
\rho(\b, \phi, \bfM) = \sum_{s \in \sS} \pi_{\b}(s) \mu_{\phi(a)}(a), \label{eq:rho2_def}
\end{align}
where $\bfM$ denotes the matrix containing all $\mu_{j}(a)$ values for power-channel pairs. For the genie under consideration, the LP from (\ref{eq:MC_LP}) is equivalent to
\begin{align}
\b^{*} = \argmax_{\b \in \sB} \rho(\b, \phi^{*}, \bfM). 
\end{align}
The expected average reward of the genie can, therefore, be defined as $\rho^{*} = \rho(\b^{*}, \phi^{*}, \bfM)$. Since the mean rate matrix $\bfM$ is unknown to the agent, we propose an online learning framework for this problem.

\subsection{Online Learning Algorithm}
Since the agent does not know the distributions of channel gains, it needs to learn the rates for various power-channel pairs, figure out $\phi^{*}$ over time and use it to make decisions about the transmit power-level at each slot. We propose an online learning algorithm called Multi-Channel LPSM (MC-LPSM) for this problem. We analyze the performance of MC-LPSM in terms of the regret as defined in equation (\ref{eq:regret}).

The MC-LPSM algorithm stores estimates of the rates for all power-level and channel pairs based on the observed values of the channel gains. Whenever the rate obtained is revealed to the agent, it can infer the instantaneous gain of the chosen channel knowing the transmit power-level. Once the instantaneous gain of a channel is known, this information can be used to update the sample-mean rate estimates of all the power-levels for that channel. The algorithm divides time into two interleaved sequences: an exploration sequence and an exploitation sequence similar to the DSEE algorithm from \cite{vakili2013}. In the exploitation sequence, the agent uses its current estimates of the rates to determine the transmit power and channel selection policies. First it selects a channel for each power-level that has the highest empirical rate for that transmit power. The sample-mean rate estimates for these power-channel pairs are, then, used to solve the LP from equation (\ref{eq:MC_LP}) with $\th$ values replacing the $\mu$ values and to obtain a power-selection policy for that slot. In the exploration sequence, the agent selects all channels in a round-robin fashion in order to learn the rates over time and chooses the transmit power-levels arbitrarily. The choice of the length of the exploration sequence balances the tradeoff between exploration and exploitation. 

Let $\sR(t)$ denote the set of time indexes that are marked as exploration slots up to time $t$. Let $\abs{\sR(t)}$ be the cardinality of the set $\sR(t)$. At any given time $t$, $m_{j}$ stores the number of times $j$-th channel has been chosen during the exploration sequence till that slot. Using these notations we present MC-LPSM in algorithm \ref{algo:MC_LPSM}. Note that MC-LPSM stores a $\th_{j}(a)$ variable for every action action-channel pair $(a,j)$ and an $m_{j}$ variable for every channel $j$. It, therefore, requires $O(MA)$ storage. 

%

\begin{algorithm}
\caption{MC-LPSM}
\label{algo:MC_LPSM}
\begin{algorithmic}[1]
\STATE  {\bf Parameters:} $w > \frac{2 B_{0}^{2}}{d^{2}}$.
\STATE  {\bf Initialization:} For all $a \in \sA$ and $j \in \{1, 2, \cdots, M \}$, $\th_{j}(a) = 0$. Also $m_{j} = 0$ for all channels $j$ and $n=0$.
\WHILE {$n < T$}
\STATE $n = n +1$;
\IF {$n \in \sR(T)$}
\STATE // Exploration sequence
\STATE { Choose channel $j = ((n-1) \mod M) + 1$ and any valid power-level as action $a$;} 
\STATE Update $\th_{j}(a)$ variables for all actions $a \in \sA$ for the chosen channel $j$:
\begin{align}
\th_{j}(a) &\leftarrow \frac{m_{j} \th_{j}(a) + f(a, x_{n})}{m_{j}+1}, \nn \\
m_{j} &\leftarrow m_{j} + 1;
\end{align}
\ELSE
\STATE // Exploitation sequence 
\IF {$n - 1 \in \sR(T)$}
\STATE Define a channel mapping $\phi$, such that $\phi(a) = \max\limits_{j} \th_{j}(a)$;
\STATE {Solve the LP from (\ref{eq:MC_LP}) with $\th_{\phi(a)}(a)$ instead of unknown $\mu^{*}(a)$ for all valid state-action pairs $(s,a)$;}
\STATE {In terms of the LP solution $\pi_{(n)}$, define $\b(s)= \argmax\limits_{a \in \sA_{s}} \pi_{(n)}(s,a), \; \forall s \in \sS$; }
\ENDIF
\STATE Given the state $s_{n}$, select the power-level $a_{n} = \b(s_{n})$ as action for transmission over channel $\phi(a_{n})$;
\ENDIF
\ENDWHILE
\end{algorithmic}
\end{algorithm} 

Note that the agent updates the $\th$ variables only during the exploration sequence when it tries different channels sequentially. Since the $\th$ variables do not change during exploitation, the agent does not have to solve the LP in all exploitation slots. During a phase of successive exploitation slots, the channel and power selection policies obtained by solving the LPs remain unchanged. The agent, therefore, needs to solve the LP at time $t$ only if the previous slot was an exploration slot. Since there are at most $\abs{\sR(T)}$ exploration slots, MC-LPSM solves at most $\abs{\sR(T)}$ number of LPs in $T$ slots.

\subsection{Regret Analysis of MC-LPSM}
Let us first define the notations used in the regret analysis. Since the agent is unaware of the matrix of expected rates $\bfM$, it stores the estimates of the expected rates in matrix $\Th$. We define $\rho(\b, \phi, \Th)$ according to equation (\ref{eq:rho2_def}) with the actual mean values replaced by their corresponding estimates in $\Th$. Let $P_{*}$ denote the transition probability matrix corresponding to the optimal state-action mapping $\b^{*}$. We further define:
\begin{align}
\g &= \max\limits_{s, s' \in \sS} \norm{P_{*}(s',\cdot) - P_{*}(s,\cdot)}_{\tv} \\
\mu_{\max} &= \max\limits_{a \in \sA} \mu^{*}(a) \\
\D_{\max} &=  \rho^{*} - \min\limits_{a \in \sA,\ j\in \{1, 2, \cdots, M \}} \mu_{j}(a) \\
\D_{3} &= \min\limits_{a \in \sA,\ j \neq \phi^{*}(a)} \left\{ \mu^{*}(a) - \mu_{j}(a) \right\} \label{eq:D3_def} \\
\D_{4} &= \rho^{*} - \underset{\b \neq \b^{*}}{\max}\; \rho(\b, \phi^{*}, \bfM) \\
B_{0} &= \sup_{x \in \sX} f(a, x) - \inf_{x \in \sX} f(a, x).
\end{align}
In terms of these notations, we provide an upper bound on the regret of MC-LPSM as follows.

\begin{theorem} \label{thm:MC_regret}
Given a constant $d \leq \min \{ \D_{3}, \D_{4} \}$, choose a constant $w > \frac{2 B_{0}^{2}}{d^{2}}$. Construct an exploration sequence as follows: for any $t > 1$, include $t$ in $\sR$ iff $\abs{\sR(t-1)} < M \lceil w \ln t \rceil$. Under this exploration sequence $\sR$, the $T$-slot expected regret of MC-LPSM algorithm is upper bounded by 
\begin{align}
\left( M \lceil w \ln T \rceil + 2 A M c_{(\frac{wd^{2}}{2 B_{0}^{2}})}\right) \left( \D_{\max} + \frac{\mu_{\max}}{1 - \g} \right), \label{eq:MC_regret}
\end{align}
where $c_{(x)} = \sum_{t=1}^{\infty} t^{-x} < \infty$ for $x > 1$.
\end{theorem}
\begin{proof}
In order to upper bound the regret of the MC-LPSM algorithm, we analyze the number of time-slots where the agent plays policy combinations other than $(\b^{*}, \phi^{*})$. Such a failure event at time $t$ corresponds to at least one of the following cases
\begin{enumerate}
    \item $t \in \sR$, i.e. the exploration of different channels,
    \item $\phi_{t} \neq \phi^{*}$ during exploitation,
    \item $\b_{t} \neq \b^{*}$ during exploitation.
\end{enumerate}

Let $N_{4}(T)$ be the total number of exploitation slots where MC-LPSM fails to find the optimal power-channel mapping $\phi^{*}$ or the optimal state-action mapping $\b^{*}$ up to time $T$. Let us define events $\sE_{1,t} = \{ \phi_{t} \neq \phi^{*} \}$ and $\sE_{2,t} = \{ \b_{t} \neq \b^{*} \}$. Now $N_{4}(T)$ can be expressed as
\begin{align}
N_{4}(T) &= \sum_{t \notin \sR, t \leq T}  \bbI \{ \sE_{1,t} \cup \sE_{2,t} \} \nn \\
&= \sum_{t \notin \sR, t \leq T} \left( \bbI \{ \sE_{1,t} \} + \bbI \{ \sE_{2,t} \cap \overline{\sE_{1,t}} \} \right).
\end{align}
We analyse the two events separately and upper bound their probabilities. 

\subsubsection{Non-Optimal Power-Channel Mapping}
We use $\th^{*}_{t}(a)$ to denote $\th_{\phi^{*}(a),t}(a)$ for all action $a$. The probability of the event $\sE_{1,t}$ can be bounded as
\begin{align}
\pr\{\phi_{t} \neq \phi^{*} \} &= \pr \{\text{For at least one action $a \in \sA$ such that:} \phi_{t}(a) \neq \phi^{*}(a)\} \nn \\
&\leq \sum_{a \in \sA} \pr \{\phi_{t}(a) \neq \phi^{*}(a)\} \nn \\
&\leq \sum_{a \in \sA} \pr \{\text{For at least one channel $j \neq \phi^{*}(a)$: } \th_{j,t}(a) \geq \th^{*}_{t}(a)\} \nn \\
&= \sum_{a \in \sA} \sum_{j \neq \phi^{*}(a)} \pr \{ \th_{j,t}(a) \geq \th^{*}_{t}(a)\}. \label{eq:event_1_a}
\end{align}
In order for the condition $\th_{j,t}(a) \geq \th^{*}_{t}(a)$ to hold, at least one of the following must hold:
\begin{align}
\th_{j,t}(a) &\geq \mu_{j}(a) + \frac{\D_{3}}{2} \label{eq:event_1_cond_1} \\
\th^{*}_{t}(a) &\leq \mu^{*}(a) - \frac{\D_{3}}{2} \label{eq:event_1_cond_2} \\
\mu^{*}(a)  &< \mu_{j}(a) + \D_{3}. \label{eq:event_1_cond_3}
\end{align}
Note that condition (\ref{eq:event_1_cond_3}) cannot hold due to the definition of $\D_{3}$. Hence we upper bound the the probabilities of the other two events. The construction of the exploration sequence guarantees that at $t \notin \sR$ each channel has been explored at least $\lceil w \ln t \rceil$ times. Since $B_{0}$ upper bounds the maximum deviation in the range of rate values over channels, we bound the probability for the event from condition (\ref{eq:event_1_cond_1}) using Hoeffding's inequality as
\begin{align}
\pr \left\{ \th_{j,t}(a) \geq \mu_{j}(a) + \frac{\D_{3}}{2} \right\} &\leq \me^{-\frac{1}{2}(\frac{\D_{3}}{B_{0}})^{2} \lceil w \ln t \rceil} \nn \\
&\leq \me^{-\frac{1}{2}(\frac{\D_{3}}{B_{0}})^{2}  w \ln t } \nn \\
&\leq t^{-\frac{w}{2}(\frac{\D_{3}}{B_{0}})^{2}}.
\end{align}
Using the Hoeffding's inequality again for the condition (\ref{eq:event_1_cond_2}), we similarly obtain
\begin{align}
\pr \left\{ \th^{*}_{t}(a) \leq \mu^{*}(a) - \frac{\D_{3}}{2} \right\} \leq t^{-\frac{w}{2}(\frac{\D_{3}}{B_{0}})^{2}}.
\end{align}
We can, therefore, express the upper bound from equation (\ref{eq:event_1_a}) as
\begin{align}
\pr\{\phi_{t} \neq \phi^{*} \} &\leq \sum_{a \in \sA} \sum_{j \neq \phi^{*}(a)} \left( \pr \left\{ \th_{j,t}(a) \geq \mu_{j}(a) + \frac{\D_{3}}{2} \right\} + \pr \left\{ \th^{*}_{t}(a) \leq \mu^{*}(a) - \frac{\D_{3}}{2} \right\} \right) \nn \\
&\leq \sum_{a \in \sA} \sum_{j \neq \phi^{*}(a)} 2 t^{-\frac{w}{2}(\frac{\D_{3}}{B_{0}})^{2}} \nn \\
&\leq 2 A (M - 1) t^{-\frac{w}{2}(\frac{\D_{3}}{B_{0}})^{2}}. \label{eq:event_1_b}
\end{align}

\subsubsection{Non-Optimal State-Action Mapping}
We analyse the event $\sE_{2,t} \cap \overline{\sE_{1,t}}$ where the LP fails to find the optimal state-action mapping $\b^{*}$ in spite of having found the optimal power-channel mapping $\phi^{*}$. 
\begin{align}
\pr\{ \sE_{2,t} \cap \overline{\sE_{1,t}} \} &= \pr\{ \b_{t} \neq \b_{t} ; \phi_{t} = \phi^{*} \} \nn \\
&= \pr \left\{ \rho(\b^{*}, \phi^{*}, \Th_{t}) \leq \rho(\b_{t}, \phi^{*}, \Th_{t})\right\}. \label{eq:event_2_a}
\end{align}
For $\rho(\b^{*}, \phi^{*}, \Th_{t}) \leq \rho(\b_{t}, \phi^{*}, \Th_{t})$ to hold, at least one of the following must hold:
\begin{align}
\rho(\b^{*}, \phi^{*}, \Th_{t}) &\leq \rho(\b^{*}, \phi^{*}, \bfM) - \frac{\D_{4}}{2} \label{eq:event2_cond_1} \\
\rho(\b_{t}, \phi^{*}, \Th_{t}) &\geq \rho(\b_{t}, \phi^{*}, \bfM) + \frac{\D_{4}}{2} \label{eq:event2_cond_2} \\
\rho(\b^{*}, \phi^{*}, \bfM) &< \rho(\b_{t}, \phi^{*}, \bfM) +  \D_{4}. \label{eq:event2_cond_3}
\end{align}
The condition from equation (\ref{eq:event2_cond_3}) cannot hold due to the definition of $\D_{4}$. We use the techniques from equations (\ref{eq:cond1_proof}) and (\ref{eq:cond2_proof}) in the proof theorem \ref{thm:non_opt_1} to upper bound the probabilities of the events of equations (\ref{eq:event2_cond_1}) and (\ref{eq:event2_cond_2}) as
\begin{align}
\pr \left\{ \rho(\b^{*}, \phi^{*}, \Th_{t}) \leq \rho(\b^{*}, \phi^{*}, \bfM) - \frac{\D_{4}}{2} \right\} &\leq \min\{S,A \} \me^{-\frac{1}{2}(\frac{\D_{4}}{B_{0}})^{2} \lceil w \ln t \rceil} \nn \\
&\leq A t^{-\frac{w}{2}(\frac{\D_{4}}{B_{0}})^{2}} \\
\pr \left\{ \rho(\b_{t}, \phi^{*}, \Th_{t}) \geq \rho(\b_{t}, \phi^{*}, \bfM) + \frac{\D_{4}}{2} \right\} &\leq  A \me^{-\frac{1}{2}(\frac{\D_{4}}{B_{0}})^{2} \lceil w \ln t \rceil} \nn \\
&\leq A t^{-\frac{w}{2}(\frac{\D_{4}}{B_{0}})^{2}}.
\end{align}
Note that these concentration bounds are different from the single channel case, as the number of observations leading to $\th_{t}$ is only $\lceil w \ln t \rceil$ in contrast to $t$ observations for the single channel. Now we update the upper bound from equation (\ref{eq:event_2_a}) as
\begin{align}
\pr\{ \sE_{2,t} \cap \overline{\sE_{1,t}} \} &\leq \pr \left\{ \rho(\b^{*}, \phi^{*}, \Th_{t}) \leq \rho(\b^{*}, \phi^{*}, \bfM) - \frac{\D_{4}}{2} \right\} + \pr \left\{ \rho(\b_{t}, \phi^{*}, \Th_{t}) \geq \rho(\b_{t}, \phi^{*}, \bfM) + \frac{\D_{4}}{2} \right\} \nn \\
&\leq 2 A t^{-\frac{w}{2}(\frac{\D_{4}}{B_{0}})^{2}}. \label{eq:event_2_b}
\end{align}

The expected number of exploitation slots, where non-optimal power and channel selection decisions are made $\bbE[N_{4}(T)]$, can be bounded using equations (\ref{eq:event_1_b}) and (\ref{eq:event_2_b}) as
\begin{align}
\bbE[ N_{4}(T)] &\leq \sum_{t=1}^{T} \left( \pr\{\sE_{1,t} \} + \pr\{ \sE_{2,t} \cap \overline{\sE_{1,t}} \} \right) \nn \\
&\leq \sum_{t=1}^{T} \left( 2 A (M - 1) t^{-\frac{w}{2}(\frac{\D_{3}}{B_{0}})^{2}} + 2A t^{ -\frac{w}{2} (\frac{\D_{4}}{B_{0}})^{2}} \right) \nn \\
&\leq 2 A M \sum_{t=1}^{T} t^{-\frac{w}{2}(\frac{d}{B_{0}})^{2}} \nn \\
&\leq 2 A M c_{(\frac{wd^{2}}{2 B_{0}^{2}})}, \label{eq:exploration_bound}
\end{align}
where $d \leq \min \{ \D_{3}, \D_{4}\}$. Since $w > \frac{2 B_{0}^{2}}{d^{2}}$, the upper bound from equation (\ref{eq:exploration_bound}) holds. The expected number of slots, where non-optimal decisions are made including exploration and exploitation sequences, is upper bounded by $M \lceil w \ln T \rceil + 2 A M c_{(\frac{wd^{2}}{2 B_{0}^{2}})}$. This implies that there can at most be $M \lceil w \ln T \rceil + 2 A M c_{(\frac{wd^{2}}{2 B_{0}^{2}})}$ phases where the optimal policies are played in succession, since any two optimal phases must have at least one non-optimal slot in between. The total expected regret of any optimal phase is bounded by $\frac{\mu_{\max}}{1 - \g}$ and the expected regret incurred during a non-optimal slot is bounded by $\D_{\max}$. The total expected $T$-slot regret of the MC-LPSM algorithm is, therefore, upper bounded by the expression (\ref{eq:MC_regret}).
\end{proof}

Note that the length of the exploration sequence specified in theorem \ref{thm:MC_regret} scales logarithmically in time. The MC-LPSM algorithm using this exploration sequence, therefore, solves $O(\ln T)$ number of LPs in $T$ slots, similar in order to the single channel Epoch-LPSM algorithm.

It must be noted that the logarithmic order regret is achievable by MC-LPSM if we know $d$, a lower bound on $\D_{3}$ and $\D_{4}$. This is required in order to define a constant $w$ that leads to the series convergence in the regret proof. If no such knowledge is available, the exploration sequence needs to be expanded in order to achieve a regret that is arbitrarily close to the logarithmic order, similar to the DSEE techniques from \cite{vakili2013}. The regret result for such an exploration sequence is as follows:
\begin{theorem} [Theorem $2$ from \cite{vakili2013}]
Let $g$ be any positive, monotonically non-decreasing sequence with $g(t) \rightarrow \infty$ as $t \rightarrow \infty$. Construct an exploration sequence as follows: for any $t > 1$, include $t$ in $\sR$ iff $\abs{\sR(t-1)} < M \lceil g(t) \ln t \rceil$. Under this exploration sequence $\sR$, the $T$-slot expected regret of MC-LPSM algorithm is $O(g(T) \ln T)$.
\end{theorem}

While this regret is not logarithmic in time, one can approach arbitrarily close to the logarithmic order by reducing the diverging rate of $g(t)$. With this construction of the exploration sequence, the MC-LPSM algorithm solves $O( g(T) \ln T)$ number of LPs in time $T$.

\subsection{Asymptotic Lower Bound}
In the multi-channel scenario, there exist one or more channels that are optimal for some transmit power levels with non-zero stationary probability. For every optimal channel $j$, there exists some state $s \in \sS$ such that $\phi^{*}(\b^{*}(s)) = j$. There may also exist arms that are either not optimal for any transmit power level or are optimal for power levels that have zero stationary probability. We now present an asymptotic lower-bound on regret of any algorithm for the multi-channel energy harvesting communications problem under certain conditions. To prove the regret bound, we first present a lower bound on the number of plays of the non-optimal channels for any algorithm. Our analysis is based on the asymptotic lower bound on the regret of the standard MAB problem by Lai and Robbins \cite{lai1985}. This MAB regret lower bound applies to the settings where the arm-distributions are characterized by a single parameter. This result was extended by Burnetas and Katehakis to distributions indexed by multiple parameters in \cite{burn1997}. In our analysis, however, we restrict ourselves to the single parameter channel-gain distributions.

Let the gain distribution of each channel be expressed by its density function $g(x; \p)$ with respect to some measure $\nu$ , where the density function $g(\cdot; \cdot)$ is known and $\p$ is an unknown parameter from some set $\Psi$. Although we consider continuous distributions here, the analysis also holds for discrete distributions where probability mass function replaces the density and the summations replace the integrals. Corresponding to a valid transmit power $a$ and parameter $\p \in \Psi$, we define the expected rate as 
\begin{align}
\mu(a; \p) = \int_{x \in \sX} f(a, x) g(x; \p) d\nu(x).
\end{align}
Let $\ttI(\p, \p')$ denote the Kullback-Leibler distance defined as
\begin{align}
\ttI(\p, \p') = \int_{x \in \sX} \left[ \ln \left( \frac{g(x; \p)}{g(x; \p')} \right) \right] g(x; \p) d\nu(x).
\end{align}
We now make following assumptions about the density and the parameter set under consideration.
\begin{enumerate}
\item[A1] Existence of mean: $\mu(a; \p) < \infty$ exists for any $\p \in \Psi$ and $a \in \sA$.
\item[A2] Denseness of $\Psi$: $\forall \p\in \Psi$, $\forall a \in \sA$ and $\forall \d >0$, $\exists \p' \in \Psi$ such that $\mu(a; \p) < \mu(a; \p') < \mu(a; \p) + \d$.
\item[A3] Positivity of distance: $0 < \ttI(\p, \p') < \infty$ whenever $\mu(a; \p) < \mu(a; \p')$ for some $a \in \sA$.
\item[A4] Continuity of $\ttI(\p, \p')$: $\forall \e > 0$, $\forall a \in \sA$ and $\forall \p, \p' \in \P$ such that $\mu(a; \p) < \mu(a; \p')$, $\exists \d = \d(a, \e, \p, \p') > 0$ for which $\abs{\ttI(\p, \p') - \ttI(\p, \p'')} < \e$ whenever $\mu(a; \p') < \mu(a; \p'') < \mu(a; \p') + \d$.
\end{enumerate}
For channel gain distributions satisfying these conditions, we present a lower bound on the number of plays of a non-optimal arm based on the techniques from \cite{lai1985}.

\begin{theorem} \label{thm:LB_count}
Assume that the density and the parameter set satisfy assumptions A1-A4. Let $\bm{\p} = (\p_{1}, \p_{2}, \cdots, \p_{M})$ be a valid parameter vector characterizing the distributions of the $M$ channels, $\bbP_{\bm{\p}}$ and $\bbE_{\bm{\p}}$ be the probability measure and expectation under $\bm{\p}$. Let $\sL$ be any allocation rule that satisfies for every $\bm{\p}$ as $T \rightarrow \infty$, $\mathfrak{R}_{\sL}(T) = o(T^{b})$ for every $b > 0$ over an MDP $\sM$. Let $N_{i}(T)$ denote the number of plays of $i$-th channel up to time $T$ by the rule $\sL$, and $\sO_{\bm{\p}}$ the index set of the optimal channels under the parameter vector $\bm{\p}$. Then for every channel $i \in \overline{\sO_{\bm{\p}}}$,
\begin{align}
\liminf\limits_{T \rightarrow \infty} \bbE_{\bm{\p}}\left[\frac{N_{i}(T)}{\ln T}\right] \geq \max_{j\in \sO_{\bm{\p}}} \frac{1}{\ttI(\p_{i},\p_{j})}. \label{eq:LB_count_result}
\end{align}
\end{theorem}
\begin{proof}
Without the loss of generality, we assume that $1 \in \overline{\sO_{\bm{\p}}}$ and $2 \in \sO_{\bm{\p}}$ for the parameter vector $\bm{\p}$. This means that $\exists a \in \sA$ such that $\m(a; \p_{2}) > \m(a; \p_{1})$ and $\m(a; \p_{2}) \geq \m(a; \p_{j})$ for $3 \leq j \leq M$. Fix any $0 < \d < 1$. By assumptions A2 and A4, we can choose $\s \in \P$ such that
\begin{align}
\m(a; \s) > \m(a; \p_{2})\ \text{and}\ \abs{ \ttI(\p_{1}, \s) - \ttI(\p_{1}, \p_{2}) } < \d \ttI(\p_{1}, \p_{2}).
\end{align}
Let us define a new parameter vector $\bm{\s} = (\s, \p_{2}, \cdots, \p_{M})$ such that under $\bm{\s}$, $1 \in \sO_{\bm{\s}}$. The basic argument is that any algorithm incurring regrets of order $o(T^{b})$ for every $b>0$ must play every channel a minimum number of times to be able to distinguish between the cases $\bm{\p}$ and $\bm{\s}$. 

Let $N_{i, j}(T)$ denote the number of times the $i$-th channel has been played up to time $T$ with power levels for which $j$-th channel was the optimal channel. We, therefore, have $N_{i}(T) = \sum_{j = 1}^{M} N_{i,j}(T)$ where $N_{i,j}(T) \geq 0$ for all $(i,j)$ pairs. We define $T_{i}(T)$ as the number of plays of the power levels for which $i$-th channel is optimal up to time $T$. This implies that the allocation rule $\sL$ plays channels other than the $i$-th channel for $T_{i}(T) - N_{i,i}(T)$ number of times where they were non-optimal. Fix $0 < b < \d$. Since $\mathfrak{R}_{\sL}(T) = o(T^{b})$ as $T \rightarrow \infty$, we have $\bbE_{\bm{\s}} [N_{j,i}(T)] = o(T^{b})$ when $i\neq j$. Hence for distributions parametrized by $\bm{\s}$, we have
\begin{align}
\bbE_{\bm{\s}}[T_{1}(T) - N_{1,1}(T)] = \sum_{j\neq 1}\bbE_{\bm{\s}} [N_{j,1}(T)] = o(T^{b}). \label{eq:LB_cond1}
\end{align}

We define a stationary distribution over channel plays under the optimal power selection policy for the MDP as 
\begin{align}
\pi_{j} = \sum_{s\in \sS}\ \sum_{a \in \sA,\ \b^{*}(a)=j} \pi^{*}(s,a).
\end{align}
Note that the optimal policies $\f^{*}$ and $\b^{*}$ are dependent on the choice of the parameter vector characterizing the channels and so is $\pi_{j}$. For channels $j \in \sO_{\bm{\s}}$, $\pi_{j} > 0$ under $\bm{\s}$. 
Since $\mathfrak{R}_{\sL}(T) = o(T^{b})$ asymptotically, we have $\bbE_{\bm{\s}}[\abs{\pi_{1} T - T_{1}(T)}] = o(T^{b})$ as $T \rightarrow \infty$. Fix $0 < c < 1$. We have 
\begin{align}
\bbP_{\bm{\s}}\left\{ T_{1}(T) \leq (1 - c) \pi_{1} T \right\} &= \bbP_{\bm{\s}}\left\{ \pi_{1} T - T_{1}(T) \geq c \pi_{1} T \right\} \nn \\
&\leq \bbP_{\bm{\s}}\left\{ \abs{\pi_{1} T - T_{1}(T)} \geq c \pi_{1} T \right\} \nn \\
&\leq \frac{\bbE_{\bm{\s}}[\abs{\pi_{1} T - T_{1}(T)}]}{c \pi_{1}T} \tag{Markov's Inequality} \\
&= o(T^{b-1}).
\end{align}
We consider another event:
\begin{align}
&\bbP_{\bm{\s}} \left\{N_{1,1}(T) < (1-\d) \frac{\ln T}{\ttI(\p_{1}, \s)}; T_{1}(T) \geq (1-c) \pi_{1} T \right\} \nn \\ &\leq \bbP_{\bm{\s}} \left\{T_{1}(T) -  N_{1,1}(T) \geq T_{1}(T) - (1-\d) \frac{\ln T}{\ttI(\p_{1}, \s)}; T_{1}(T) \geq (1-c) \pi_{1} T \right\} \nn \\
&\leq \bbP_{\bm{\s}} \left\{T_{1}(T) -  N_{1,1}(T) \geq (1-c) \pi_{1} T - (1-\d) \frac{\ln T}{\ttI(\p_{1}, \s)}; T_{1}(T) \geq (1-c) \pi_{1} T \right\} \nn \\
&\leq \bbP_{\bm{\s}} \left\{T_{1}(T) -  N_{1,1}(T) \geq (1-c) \pi_{1} T - (1-\d) \frac{\ln T}{\ttI(\p_{1}, \s)} \right\} \nn \\
&\leq \frac{\bbE_{\bm{\s}}[T_{1}(T) - N_{1,1}(T)]}{(1-c) \pi_{1}T - O(\ln T)} \tag{Markov's Inequality} \\
&= o(T^{b-1}).
\end{align}

For the event $\bbP_{\bm{\s}}  \left\{N_{1}(T) < (1-\d) \frac{\ln T}{\ttI(\p_{1}, \s)} \right\}$, we have
\begin{align}
&\bbP_{\bm{\s}}  \left\{N_{1}(T) < (1-\d) \frac{\ln T}{\ttI(\p_{1}, \s)} \right\} \nn \\
&\leq \bbP_{\bm{\s}} \left\{N_{1,1}(T) < (1-\d) \frac{\ln T}{\ttI(\p_{1}, \s)} \right\} \nn \\
&\leq \bbP_{\bm{\s}}\left\{ T_{1}(T) \leq (1 - c) \pi_{1} T \right\} + \bbP_{\bm{\s}} \left\{N_{1,1}(T) < (1-\d) \frac{\ln T}{\ttI(\p_{1}, \s)}; T_{1}(T) \geq (1-c) \pi_{1} T \right\} \nn \\
& = o(T^{b-1}). \label{eq:LB_cond2}
\end{align}

Note that the allocation rule $\sL$ only knows the channel gain realizations of the arms it has played, it does not have the exact distributional knowledge. Let $Y_{1}, Y_{2}, \cdots$ denote the successive realizations of the $1$-st channel's gains. We define $L_{m} = \sum_{k=1}^{m} \ln \frac{g(Y_{k};\p_{1})}{g(Y_{k};\s)}$ and an event $\sE_{T}$ as 
\begin{align}
\sE_{T} = \left\{N_{1}(T) < (1-\d) \frac{\ln T}{\ttI(\p_{1}, \s)}\ \text{and}\ L_{N_{1}(T)} \leq (1-b) \ln T \right\}.
\end{align}
From the inequality in (\ref{eq:LB_cond2}), we have
\begin{align}
\bbP_{\bm{\s}} \left\{ \sE_{T} \right\} = o(T^{b-1}). \label{eq:LB_cond3}
\end{align}
Note the following relationship
\begin{align}
&\bbP_{\bm{\s}} \left\{ N_{1}(T)=n_{1}, \cdots, N_{M}(T) = n_{M}\ \text{and}\  L_{n_{1}} \leq (1-b) \ln T \right\} \nn \\
&= \int_{\left\{ N_{1}(T)=n_{1}, \cdots, N_{M}(T) = n_{M}\ \text{and}\  L_{n_{1}} \leq (1-b) \ln T \right\}} \prod_{k=1}^{n_{1}} \frac{g(Y_{k};\s)}{g(Y_{k};\p_{1})} d \bbP_{\bm{\p}} \nn \\
&= \int_{\left\{ N_{1}(T)=n_{1}, \cdots, N_{M}(T) = n_{M}\ \text{and}\  L_{n_{1}} \leq (1-b) \ln T \right\}} \me^{-L_{n_{1}}} d \bbP_{\bm{\p}} \nn \\
&\geq \me^{-(1-b)\ln T} \bbP_{\bm{\p}} \left\{ N_{1}(T)=n_{1}, \cdots, N_{M}(T) = n_{M}\ \text{and}\  L_{n_{1}} \leq (1-b) \ln T \right\} \nn \\
&= T^{-(1-b)} \bbP_{\bm{\p}} \left\{ N_{1}(T)=n_{1}, \cdots, N_{M}(T) = n_{M}\ \text{and}\  L_{n_{1}} \leq (1-b) \ln T \right\}. \label{eq:LB_cond4}
\end{align}
This result rests on the assumption that the allocation rule $\sL$ can only depend on the channel gain realizations it has observed by playing and possibly on some internal randomization in the rule. Note that $\sE_{T}$ is a disjoint union of the events of the form $\left\{ N_{1}(T)=n_{1}, \cdots, N_{M}(T) = n_{M}\ \text{and}\  L_{n_{1}} \leq (1-b) \ln T \right\}$ with $n_{1} +\cdots + n_{M} = T$ and $n_{1} <(1-\d) \frac{\ln T}{\ttI(\p_{1}, \s)} $. It now follows from equations (\ref{eq:LB_cond3}) and (\ref{eq:LB_cond4}) that as $T \rightarrow \infty$:
\begin{align}
\bbP_{\bm{\p}} \left\{ \sE_{T} \right\} \leq T^{1-b} \bbP_{\bm{\s}} \left\{ \sE_{T} \right\} \rightarrow 0. \label{eq:LB_cond5}
\end{align}

By the strong law of large numbers, $\frac{L_{m}}{m} \rightarrow \ttI(\p_{1}, \s) > 0$ and $\max\limits_{k \leq m} \frac{L_{k}}{m} \rightarrow \ttI(\p_{1}, \s)$ almost surely under $\bbP_{\bm{\p}}$. Since $1-b > 1-\d$, it follows that as $T \rightarrow \infty$:
\begin{align}
\bbP_{\bm{\p}} \left\{ L_{k} > (1-b) \ln T \ \text{for some}\ k < (1-\d) \frac{\ln T}{\ttI(\p_{1}, \s)} \right\} \rightarrow 0. \label{eq:LB_cond6}
\end{align}
From equations (\ref{eq:LB_cond5}) and (\ref{eq:LB_cond6}), we conclude that 
\begin{align}
\lim_{T \rightarrow \infty} \bbP_{\bm{\p}} \left\{N_{1}(T) <  \frac{(1-\d) \ln T}{\ttI(\p_{1}, \s)} \right\} = 0. \nn
\end{align}
In other words,
\begin{align}
\lim_{T \rightarrow \infty} \bbP_{\bm{\p}}\left\{N_{1}(T) < \frac{(1-\d) \ln T}{(1+\d) \ttI(\p_{1}, \p_{2})} \right\}= 0. \nn
\end{align}
This implies that 
\begin{align}
\liminf\limits_{T \rightarrow \infty} \bbE_{\bm{\p}}\left[\frac{N_{1}(T)}{\ln T}\right] \geq \frac{1}{\ttI(\p_{1},\p_{2})}. \label{eq:LB_cond7}
\end{align}

Note that we only considered one optimal arm above. Results like equation (\ref{eq:LB_cond7}) hold for all optimal arms. By combining the lower bounds for a fixed non-optimal arm, we get the result in equation (\ref{eq:LB_count_result}).
\end{proof}

\begin{theorem} \label{thm:LB_regret}
Assume that the density and the parameter set satisfy assumptions A1-A4. Let $\bm{\p}$ denote the parameter vector whose $j$-th entry is $\p_{j}$ and $\sO$ denote the index set of optimal channels. Let $\sL$ be any allocation rule that satisfies for every $\bm{\p}$ as $T \rightarrow \infty$, $\mathfrak{R}_{\sL}(T) = o(T^{b})$ for every $b > 0$ over an MDP $\sM$. Then the regret of $\sL$ satisfies
\begin{align}
\liminf\limits_{T \rightarrow \infty} \frac{\mathfrak{R}_{\sL}(T)}{\ln T} \geq \D_{3} \sum_{i \in \overline{\sO}} \left( \max_{j\in \sO} \frac{1}{\ttI(\p_{i},\p_{j})} \right). \label{eq:LB_regret}
\end{align}
\end{theorem}
\begin{proof}
Define a hypothetical allocation $\sL'$ based on $\sL$ such that whenever $\sL$ plays a non-optimal channel for some power level during its execution, $\sL'$ plays the optimal channel corresponding to the same power level. It follows $\sL$ in rest of the slots. If $N'$ denotes the count variables corresponding to $\sL'$, then for $i \in \overline{\sO}$ we have $N'_{i}(T) = 0$ and for $j \in \sO$ we have
\begin{align}
N'_{j}(T) = N_{j}(T) + \sum_{i \in \overline{\sO}} N_{i,j}(T).
\end{align}

According to equation (\ref{eq:D3_def}), $\D_{3}$ is the minimum expected gap between the optimal rate of any power level and the rate for any other channel at that power. Using $\D_{3}$, we relate the regrets of $\sL$ and $\sL'$ as
\begin{align}
\mathfrak{R}_{\sL}(T) &\geq \mathfrak{R}_{\sL'}(T) + \bbE\left[ \sum_{i \in \overline{\sO}} N_{i}(T)\right] \D_{3} \nn \\
&\geq \D_{3} \sum_{i \in \overline{\sO}} \bbE\left[  N_{i}(T)\right].
\end{align}
Using theorem \ref{thm:LB_count}, we get the equation (\ref{eq:LB_regret}).
\end{proof}

Theorem \ref{thm:LB_regret} implies that when the gain distributions characterized by a single parameter for each channel and follow assumptions A1-A4, any algorithm with $\mathfrak{R}(T) = o(T^{b})$ for every $b > 0$ must play the non-optimal channels at least $\Om(\ln T)$ times asymptotically. In the presence of non-optimal channels, an asymptotic regret of $\Om(\ln T)$ is inevitable for any algorithm. Hence we conclude that our MC-LPSM algorithm is asymptotically order optimal when the system contains non-optimal channels. 

\section{Cost Minimization Problems} \label{sec:packet}
We have considered reward maximization problems for describing our online learning framework. This framework can also be applied to average cost minimization problems in packet scheduling with power-delay tradeoff as shown in figure \ref{fig:en_delay}. We describe this motivating example and the minor changes required in our algorithms.
 
\begin{figure}
    \centering
    \includegraphics[width=0.8\textwidth]{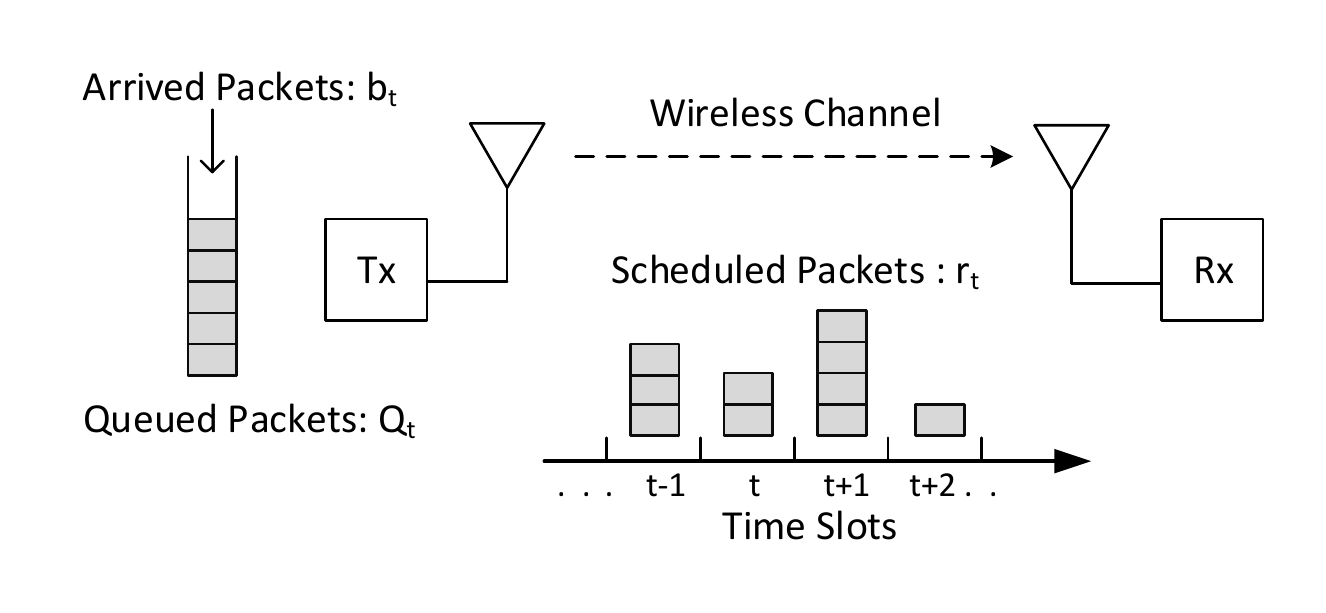}
    \caption{Packet scheduling over a wireless channel}
    \label{fig:en_delay}
\end{figure}

Consider a time-slotted communication system where a sender sends data packets to a receiver over a stochastically varying channel with unknown distribution. Such a communication system has been studied previously in \cite{yanting2015} assuming the channel to be non-stochastically varying over time. In our setting, the arrival of data packets is also stochastic with a known distribution. The sender can send multiple packets at a higher cost, or can defer some for latter slots while incurring a fixed delay penalty per packet for every time-slot it spends in the sender's queue. Let $Q_{t}$ denote the number of packets in the queue at time $t$ and $r_{t}(\leq Q_{t})$ be the number of packets transmitted by the sender during the slot. Hence, $Q_{t} - r_{t}$ number of packets get delayed. The sender's queue gets updated as
\begin{equation}
    Q_{t+1} = \min\{Q_{t} - r_{t} + b_{t}, Q_{\max} \}, \label{eq:queue_update}
\end{equation}
where $b_{t}$ is the number of new packet arrivals in $t$-th slot and $Q_{\max}$ is the maximum queue size possible. Since the data-rate is modelled according to equation (\ref{eq:rate_eq}), the power cost incurred during the $t$-th slot by transmitting $r_{t}$ packets over the channel becomes $w_{p} X_{t} 2^{r_{t}/ B}$, where $w_{p}$ is a constant known to the sender and $X_{t}$ is the instantaneous channel gain-to-noise ratio that is assumed to be i.i.d. over time. Assuming $w_{d}$ as the unit delay penalty, during the slot the sender incurs an effective cost
\begin{equation}
    C_{t} = w_{d} (Q_{t} - r_{t}) + w_{p} X_{t} 2^{r_{t} / B}.   \label{eq:cost_eq}
\end{equation}
This problem also represents an MDP where the queue size is the state and the number of packets transmitted is the action taken. The goal of this problem is to schedule transmissions $r_{t}$ sequentially and minimize the expected average cost over time
\begin{equation}
    \lim_{T\rightarrow \infty} \frac{1}{T} \bbE \left[ \sum_{t=1}^{T}  C_{t} \right]. 
\end{equation}
Note that the cost from equation (\ref{eq:cost_eq}) used in this scenario is also a function of the state unlike the problem of energy harvesting communications.

The presented algorithms LPSM and Epoch-LPSM also apply to cost minimization problems with minor changes. If $\rho(\b, \bfM)$ denotes the average expected cost of the policy $\b$, then $\rho^{*} = \min_{\b \in \sB} \rho(\b, \bfM)$. Using this optimal mean cost as the benchmark, we define the cumulative regret of a learning algorithm after $T$ time-slots as
\begin{equation}
    \mathfrak{R}(T) \defeq \bbE\left[ \sum_{t=0}^{T-1} C_{t}\right] - T \rho^{*}. \label{eq:cost_regret}
\end{equation}
In order to minimize the regret for this problem, the LP from (\ref{eq:LP}) needs to be changed from a maximization LP to a minimization LP. With these changes to the algorithms, all the theoretical guarantees still hold with the constants defined accordingly.

\section{Numerical Simulations} \label{sec:simu}
We perform simulations for the power allocation problem with $\sS= \{0, 1, 2, 3, 4\}$ and $\sA = \{0, 1, 2, 3, 4 \}$. Note that each state $s_{t}$ corresponds to $Q_{t}$ from equation (\ref{eq:battery_update}) with $Q_{\max} = 4$ and $a_{t}$ corresponds to the transmit power $q_{t}$ from equation (\ref{eq:rate_eq}). The reward function is the rate function from equation (\ref{eq:rate_eq}) and the channel gain is a scaled Bernoulli random variable with $\pr \{ X = 10\} = 0.2$ and $\pr \{ X = 0\} = 0.8$. The valid actions $\sA_{s}$ and the optimal action $\b^{*}(s)$ for each state $s$ are shown in table \ref{tab:state_action}. We use CVXPY \cite{cvxpy} for solving the LPs in our algorithms. For the simulations in figure \ref{fig:regret1}, we use $n_{0} = 2$ and $\eta = 10$, and plot the average regret performance over $10^{3}$ independent runs of different algorithms. Here, the naive policy never uses the battery, i.e. it uses all the arriving power for the current transmission. Playing such a fixed non-optimal policy causes linearly growing regret over time. Note that the optimal policy also incurs a regret because of the corresponding Markov chain not being at stationarity. We observe that LPSM follows the performance of the optimal policy with the difference in regret stemming from the first few time-slots when the channel statistics are not properly learnt and thus LPSM fails to find the optimal policy. As the time progresses, LPSM finds the optimal policy and its regret follows the regret pattern of the optimal policy. In Epoch-LPSM with $n_{0} = 2$ and $\eta = 10$, the agent solves the LP at $t=1$ and $t=2$. Its LP solution at $t=2$ is followed for the first epoch and thus the regret grows linearly till $t=19$. At $t=20$, a new LP is solved which often leads to the optimal policy and the regret contribution from latter slots, therefore, follows the regret of the optimal policy. It must be noted that Epoch-LPSM solves only 3 LPs during these slots, while LPSM solves $99$ LPs. Epoch-LPSM, therefore, reduces the computational requirements substantially while incurring a slightly higher cumulative regret. 

\begin{figure}
    \centering
    \includegraphics[width=0.7\textwidth]{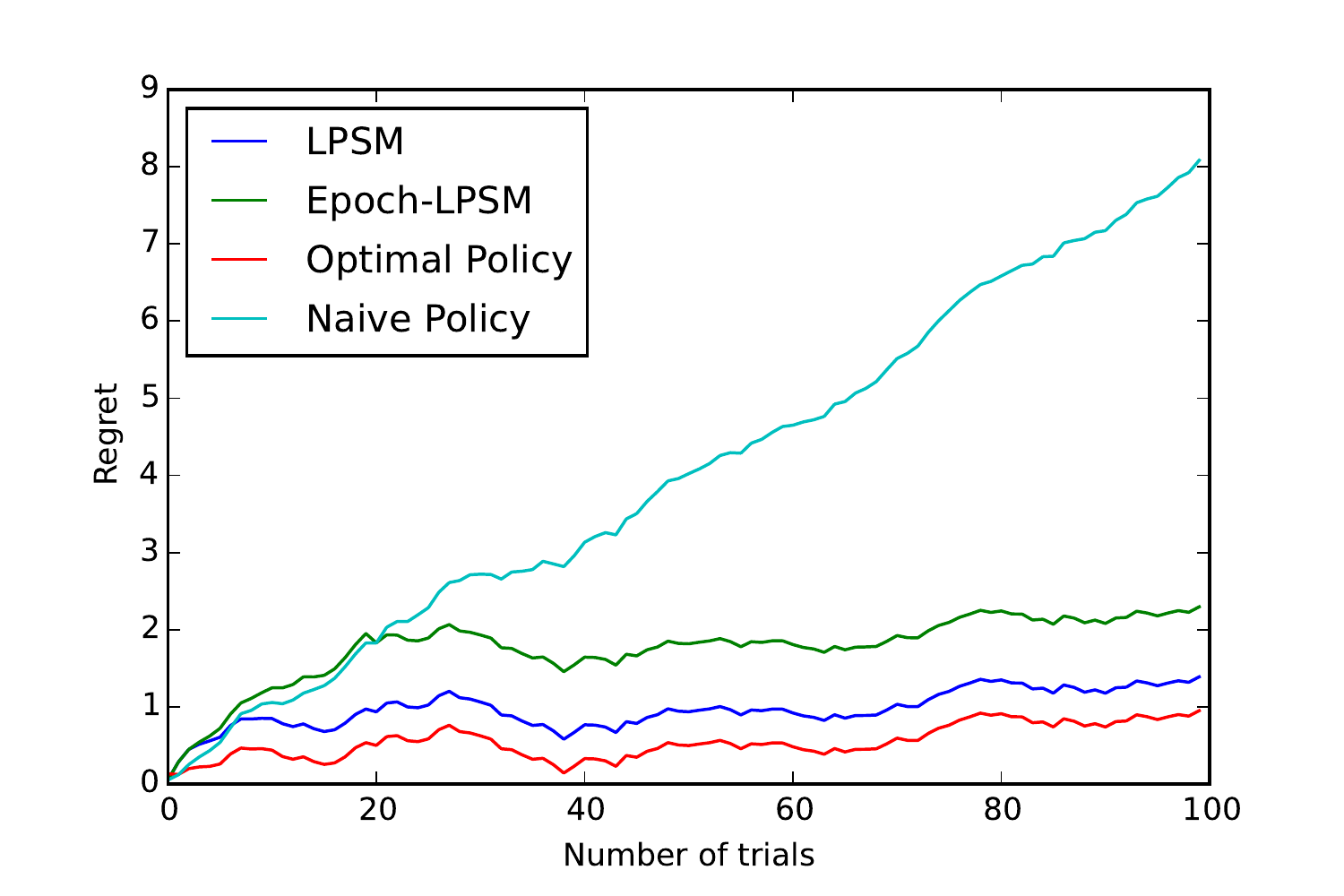}
    \caption{Regret performance of LPSM algorithms.}
    \label{fig:regret1}
\end{figure}

\begin{figure}
    \centering
    \includegraphics[width=0.7\textwidth]{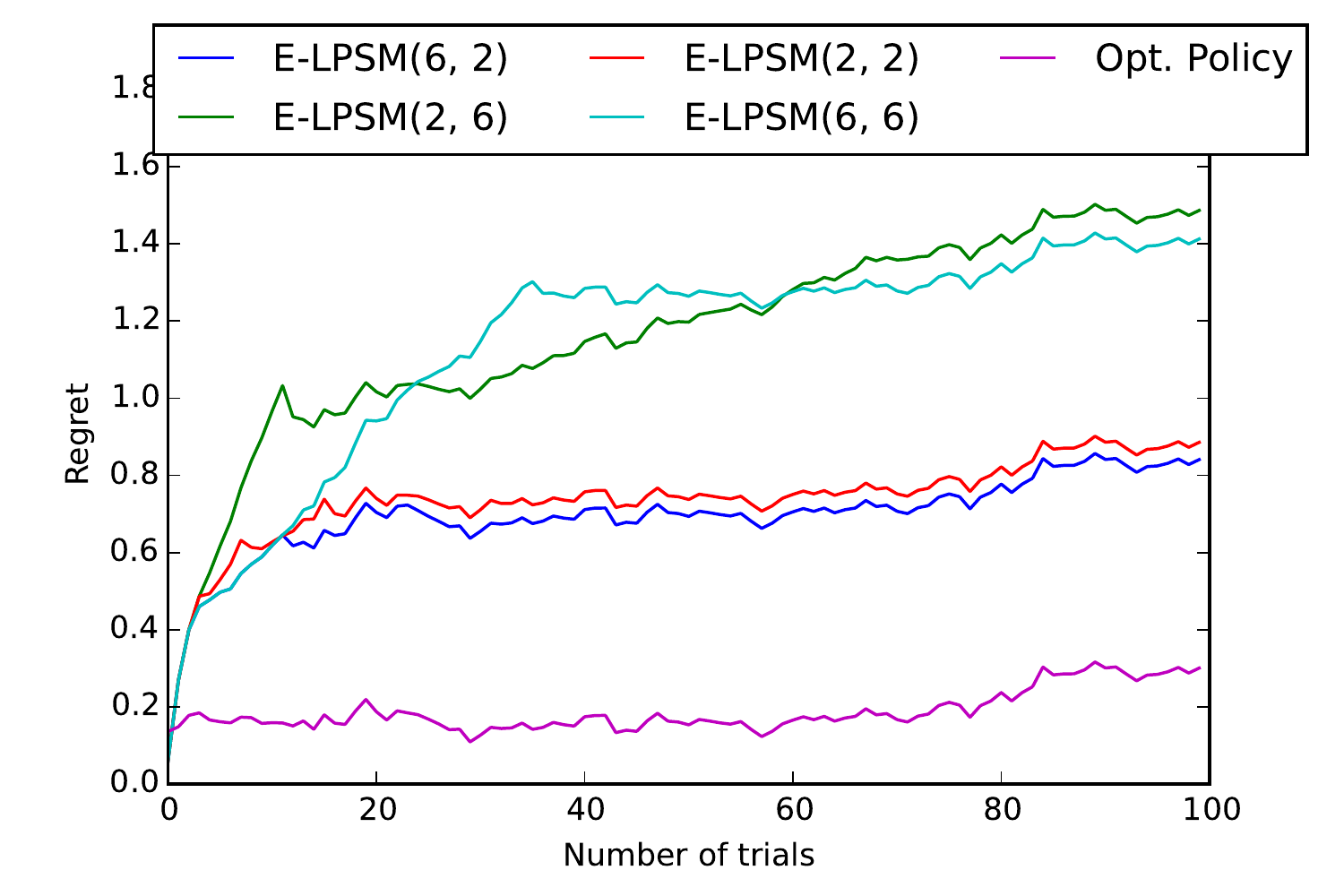}
    \caption{Effect of the parameters $n_{0}$ and $\eta$ on the regret of Epoch-LPSM.}
    \label{fig:regret2}
\end{figure}

\begin{table}[ht]
    \centering
    \caption{Actions for each state}
    \begin{tabular}{c|l | c}
        \hline
        $s$ & $\sA_{s}$ & $\b^{*}(s)$ \\
        \hline
         0 & $\{ 0\}$ & 0\\
         1 & $\{ 1 \}$ & 1\\
         2 & $\{1, 2\}$ & $1$\\
         3 & $\{1, 2, 3\}$ & $2$\\
         4 & $\{1, 2, 3, 4\}$ & $3$\\
        \hline
    \end{tabular}
    \label{tab:state_action}
\end{table}

In figure \ref{fig:regret2}, we plot the average regret performance over $10^{4}$ independent runs of Epoch-LPSM for the previous system with different $n_{0}$ and $\eta$ value pairs. As the value of $\eta$ increases for a fixed $n_{0}$, the length of the epochs increases and thus the potential non-optimal policies are followed for longer epochs. Larger the value of $n_{0}$, better are the policies played in the initial slots where the agent solves the LP at each time. These intuitions are consistent with figure \ref{fig:regret2}, where Epoch-LPSM with $n_{0}=6$ and $\eta=2$ has the lowest regret and the one with $n_{0}=2$ and $\eta=6$ has the highest of the lot. We see the regret vs computation tradeoff in action, as the decrease in computation by increasing $\eta$ or decreasing $n_{0}$ leads to larger regrets. We observe that $\eta$ has more impact on the regret than $n_{0}$. Notice that there are changes in the regret trends of Epoch-LPSM at $t=n_{0} \eta^{m}$ for small $m$, because these are the slots where a new LP is solved by MC-LPSM. Once the optimal policy is found by the algorithm, its regret in latter slots follows the trend of the optimal policy.

\begin{figure}
    \centering
    \includegraphics[width=0.7\textwidth]{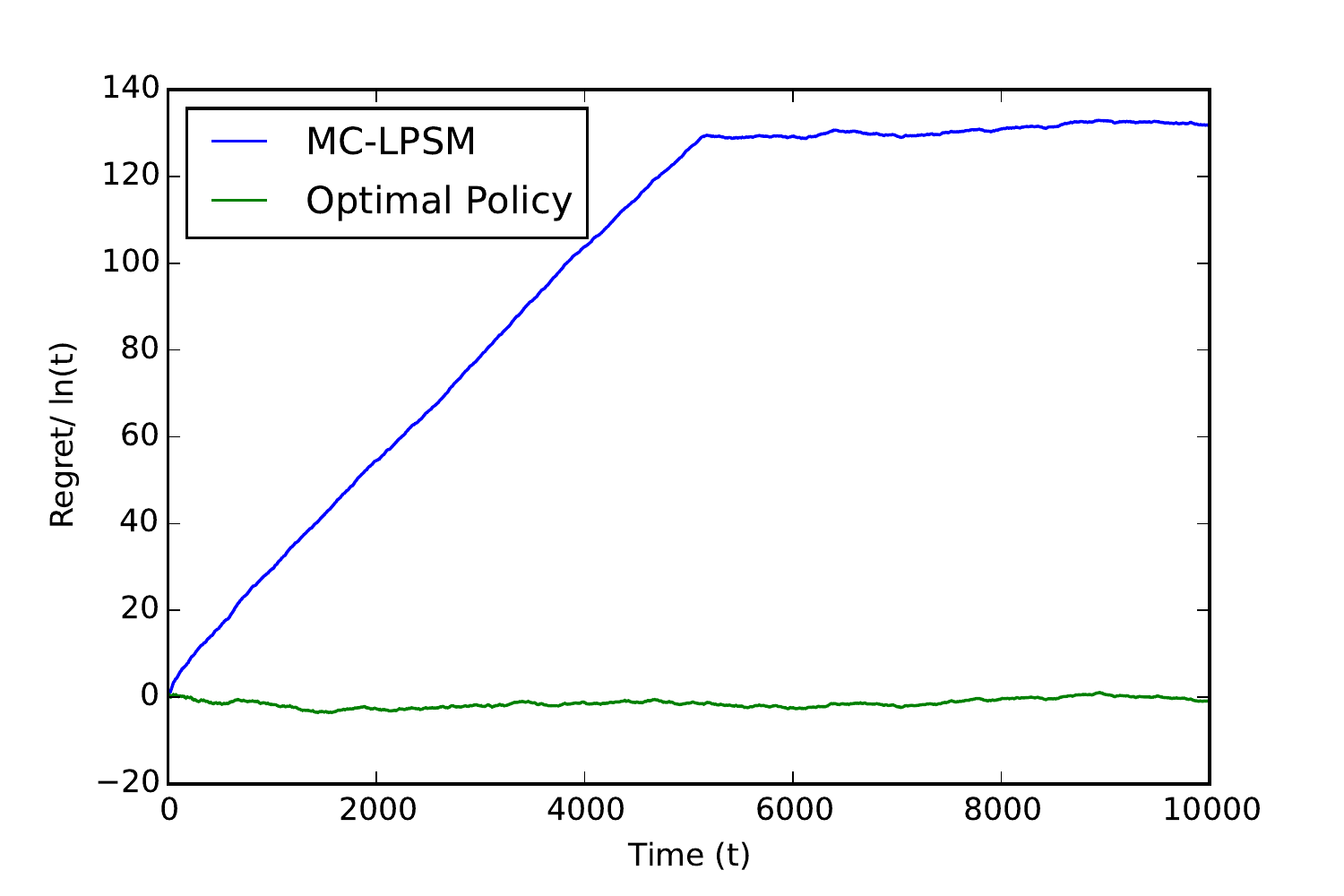}
    \caption{Regret performance of the MC-LPSM algorithm.}
    \label{fig:regret3}
\end{figure}

In figure \ref{fig:regret3}, we plot the regret performance of MC-LPSM for a system with $2$ communication channels. The channels are scaled Bernoulli random variables where the gain of the first follows $\pr \{ X = 10\} = 0.5$ and $\pr \{ X = 0\} = 0.5$, while that of the other follows $\pr \{ X = 22\} = 0.4$ and $\pr \{ X = 0\} = 0.6$. The optimal actions selection policy is same as the previous case, while the optimal mapping of transmit-power to the channels is $\phi^{*}(1) = 2$, $\phi^{*}(2) = 2$, $\phi^{*}(3) = 1$ and $\phi^{*}(4) = 1$. For MC-LPSM, we set $w= 300$ and plot the regret divided by the logarithm of the time index, averaged over $100$ realizations, in figure \ref{fig:regret3}. We notice that whenever the exploration slots are densely packed, the regret grows linearly as non-optimal policies are potentially played during exploration. As the exploration need gets satisfied over time, the agent solves the LP based on its rate estimates from exploration phases and the regret contribution from the exploitation remains bounded. The regret divided by the logarithm of time, therefore, saturates to a constant value as expected.

\section{Conclusion} \label{sec:conclusion}
We have considered the problem of power allocation over a stochastically varying channel with unknown distribution in an energy harvesting communication system. We have cast this problem as an online learning problem over an MDP. If the transition probabilities and the mean rewards associated with the MDP are known, the optimal policy maximizing the average expected reward over time can be determined by solving an LP specified in the paper. Since the agent is only assumed to know the distribution of the harvested energy, it needs to learn the rewards of the state-action pairs over time and make its decisions based on the learnt behaviour. For this problem, we have proposed two online learning algorithms: LPSM and Epoch-LPSM, which both solve the LP using the sample mean estimates of the rewards instead of the unknown mean rewards. The LPSM algorithm solves the LP at each time-slot using the updated estimates, while the Epoch-LPSM only solves the LP at certain pre-defined time-slots parametrized by $n_{0}$ and $\eta$ and thus, saves a lot of computation at the cost of an increased regret. We have shown that the regrets incurred by both these algorithms are bounded from above by constants. The system designers can, therefore, analyze the regret versus computation tradeoff and tune the parameters $n_{0}$ and $\eta$ based on their performance requirements. Through the numerical simulations, we have shown that the regret of LPSM is very close to that of the optimal policy. We have also analyzed the effect of the parameters $n_{0}$ and $\eta$ on the regret the Epoch-LPSM algorithm which approaches the regret of the optimal policy for small $\eta$ values and large $n_{0}$.

For the case of multiple channels, there is an extra layer of decision making to select a channel for transmission in each slot. For this problem, we have extended our approach and proposed the MC-LPSM algorithm. MC-LPSM separates the exploration of different channels to learn their rates from the exploitation where these rate-estimates of different channels for different channels are used to obtain a power selection policy and a channel selection policy in each slot. We have proved a regret upper bound that scales logarithmically in time and linearly in the number of channels. We have also shown that the total computational requirement of MC-LPSM also scales similarly. In order to show the asymptotic order optimality of our MC-LPSM algorithm, we have proved an asymptotic regret lower bound of $\Om(\ln T)$ for any algorithm under certain conditions.

While we have considered the reward maximization problem in energy harvesting communications for our analysis, we have shown that these algorithms also work for the cost minimization problems in packet scheduling with minor changes.

\bibliography{MAB_list}{}
\bibliographystyle{ieeetr}

\appendices

\section{Technical Lemmas} \label{apx:lemmas}
\begin{lemma}[Hoeffding's Concentration Inequality from \cite{hoeffding1963}] \label{lem:hoeffding}
Let $Y_{1}, ..., Y_{n}$ be i.i.d. random variables with mean $\m$ and range $[0,1]$. Let $S_{n} = \sum\limits_{t=1}^{n} Y_{t}$. Then for all $\a \geq 0$
\begin{align}
\pr \{ S_{n} \geq n \m + \a \} & \leq \me^{-2\a^{2}/n}  \nn \\
\pr \{ S_{n} \leq n \m - \a \} & \leq \me^{-2\a^{2}/n} \nn.
\end{align}
\end{lemma}

\section{Analysis of Markov Chain Mixing} \label{apx:MC}
We briefly introduce the tools required for the analysis of Markov chain mixing (see \cite{levin2009markov}, chapter $4$ for a detailed discussion). The total variation (TV) distance between two probability distributions $\phi$ and $\f'$ on sample space $\Omega$ is defined by 
\begin{equation}
    \norm{\phi - \f'}_{\tv} = \max_{\sE \subset \Omega} \abs{\phi(\sE) - \f'(\sE)}.
\end{equation}
Intuitively, it means the TV distance between $\phi$ and $\f'$ is the maximum difference between the probabilities of a single event by the two distributions. The TV distance is related to the $L_{1}$ distance as follows
\begin{equation}
    \norm{\phi - \f'}_{\tv} = \frac{1}{2} \sum_{\o \in \Omega} \abs{\phi(\o) - \f'(\o)}.
\end{equation}

We wish to bound the maximal distance between the stationary distribution $\pi$ and the distribution over states after $t$ steps of a Markov chain. Let $P^{(t)}$ be the $t$-step transition matrix with $P^{(t)}(s, s')$ being the transition probability from state $s$ to $s'$ of the Markov chain in $t$ steps and $\sP$ be the collection of all probability distributions on $\Om$. Also let $P^{(t)}(s,\cdot)$ be the row or distribution corresponding to the initial state of $s$. Based on these notations, we define a couple of useful $t$-step distances as follows:
\begin{align}
d(t) &\defeq \max_{s \in \sS} \norm{\pi - P^{(t)}(s,\cdot)}_{\tv} = \sup_{\phi \in \sP} \norm{\pi - \phi P^{(t)}}_{\tv}, \label{eq:def_d}\\
\hat{d}(t) &\defeq \max_{s, s' \in \sS} \norm{P^{(t)}(s',\cdot) - P^{(t)}(s,\cdot)}_{\tv} = \sup_{\f', \phi \in \sP} \norm{\f' P^{(t)} - \phi P^{(t)}}_{\tv}. \label{eq:def_d_hat}
\end{align}
For irreducible and aperiodic Markov chains, the distances $d(t)$ and $\hat{d}(t)$ have following special properties:
\begin{lemma}[\cite{levin2009markov}, lemma $4.11$] \label{lem:tv1}
For all $t > 0$, $d(t) \leq \hat{d}(t) \leq 2 d(t)$.
\end{lemma}
\begin{lemma} [\cite{levin2009markov}, lemma $4.12$] \label{lem:tv2}
The function $\hat{d}$ is sub-multiplicative: $\hat{d}(t_{1} + t_{2}) \leq \hat{d}(t_{1}) \hat{d}(t_{2})$.
\end{lemma} 

These lemmas lead to following useful corollary:
\begin{corollary} \label{cor:TV_dist}
For all $t \geq 0$, $d(t) \leq \hat{d}(1)^{t}$.
\end{corollary}

Consider an MDP with optimal stationary policy $\b^{*}$. Since the MDP might not start at the stationary distribution $\pi^{*}$ corresponding to the optimal policy, even the optimal policy incurs some regret as defined in equation (\ref{eq:regret}). We characterize this regret in the following theorem. 

\begin{theorem}[Regret of Optimal Policy] \label{thm:opt_regret}
For an ergodic MDP, the total expected regret of the optimal stationary policy with transition probability matrix $P_{*}$ is upper bounded by $(1-\g)^{-1} \mu_{\max}$, where $\g = \max\limits_{s, s' \in \sS} \norm{P_{*}(s',\cdot) - P_{*}(s,\cdot)}_{\tv}$ and $\mu_{\max} = \max\limits_{s \in \sS, a \in \sA} \mu(s,a)$.
\end{theorem}
\begin{proof}
Let $\phi_{0}$ be the initial distribution over states and $\phi_{t} = \phi_{0} P_{*}^{(t)}$ be such distribution at time $t$ represented as a row vectors. Also, let $\mu^{*}$ be a row vector with the entry corresponding to state $s$ being $\mu(s, \b^{*}(s))$. We use $d^{*}(t)$ and $\hat{d}^{*}(t)$ to denote the $t$-step distances from equations (\ref{eq:def_d}) and (\ref{eq:def_d_hat}) for the optimal policy. Ergodicity of the MDP ensures that the Markov chain corresponding to the optimal policy is irreducible and aperiodic, and thus lemmas \ref{lem:tv1} and \ref{lem:tv2} hold. The regret of the optimal policy, therefore, gets simplified as:
\begin{align}
\mathfrak{R}^{*}(\phi_{0}, T) &= T \rho^{*} - \sum_{t=0}^{T-1} \phi_{t} \cdot \mu^{*} \nn \\
&= T (\pi^{*} \cdot \mu^{*}) - \sum_{t=0}^{T-1} \phi_{t} \cdot \mu^{*} \nn \\
&= \sum_{t=0}^{T-1} (\pi^{*} - \phi_{t}) \cdot \mu^{*} \nn \\
&\leq \sum_{t=0}^{T-1} (\pi^{*} - \phi_{t})_{+} \cdot \mu^{*} \tag{Negative entries ignored} \\
&= \sum_{t=0}^{T-1} \sum_{s \in \sS} (\pi^{*}(s) - \phi_{t}(s))_{+} \mu^{*}(s) \nn \\
&\leq \mu_{\max} \sum_{t=0}^{T-1} \sum_{s \in \sS} (\pi^{*}(s) - \phi_{t}(s))_{+} \nn \\
&= \mu_{\max} \sum_{t=0}^{T-1}  \norm{\pi^{*} - \phi_{0} P_{*}^{(t)}}_{\tv} \nn \\
&\leq \mu_{\max} \sum_{t=0}^{T-1}  d^{*}(t) \nn \\
&\leq \mu_{\max} \sum_{t=0}^{T-1}  \left(\hat{d}^{*}(1)\right)^{t} \tag{From corollary \ref{cor:TV_dist}} \\
&= \mu_{\max} \sum_{t=0}^{T-1}  \g^{t} \nn \\
&\leq \mu_{\max}  \frac{1}{1- \g}. \nn
\end{align}
Note that this regret bound is independent of the initial distribution over the states.
\end{proof}

\end{document}